\documentclass{article}

\usepackage{iclr2027_conference,times}

\usepackage[utf8]{inputenc}
\usepackage[T1]{fontenc}
\usepackage{hyperref}
\usepackage{url}
\usepackage{booktabs}
\usepackage{amsfonts}
\usepackage{nicefrac}
\usepackage{microtype}
\usepackage{xcolor}
\usepackage{graphicx}
\usepackage{wrapfig}
\usepackage{adjustbox}

\usepackage{amsmath}
\usepackage{booktabs}
\usepackage{makecell}
\usepackage{array}
\usepackage{tikz}
\usepackage{bbm}
\usepackage{amsthm}
\usepackage{amssymb}
\usepackage{algorithm}
\usepackage{algorithmic}
\usepackage{enumitem}
\usepackage{cleveref}

\usepackage{multirow}
\usepackage{graphicx}

\newcommand{\mvtec}{\textbf{MVTec AD}}
\newcommand{\chaos}{\textbf{CHAOS}}

\newcommand{\ffhq}{\textbf{FFHQ}}
\newcommand{\xray}{\textbf{DVXray}}

\newcommand{\name}{LACToS}

\newcommand{\R}{\mathbb{R}}

\newcommand{\Zpixel}{\mathbb{Z}_{256}}

\newcommand{\J}{\mathcal{J}}
\newcommand{\K}{\mathcal{K}}
\newcommand{\M}{M}
\newcommand{\E}{\mathbb{E}}

\newcommand{\xstar}{x^{\star}}
\newcommand{\mstarpixel}{m^{\star}}
\newcommand{\mhatpixel}{\hat{m}}

\newcommand{\x}{x}
\newcommand{\xj}{x_{j}}

\newcommand{\xstarj}{x^{\star}_{j}}
\newcommand{\xstarnotj}{x^{\star}_{-j}}

\newcommand{\xmj}{x^{m}_{j}}

\newcommand{\pj}[1]{p(#1|x_{-j}^\star)}
\newcommand{\pk}[1]{p(#1|x_{-S_k}^\star)}

\newcommand{\qk}{q_{k}^{\mathrm{inp}}}
\newcommand{\pden}{p^{\mathrm{den}}}
\newcommand{\pinpk}{p^{\mathrm{inp}}_k}

\newcommand{\entropy}{\mathbb{H}}       

\newcommand{\spxs}{\mathcal{S}(x^{\star})}           
\newcommand{\spx}[1]{S_{#1}}                        
\newcommand{\spxk}{S_{k}}

\newcommand{\xspxk}{x_{S_{k}}}
\newcommand{\xspxnotk}{x_{-S_{k}}}
\newcommand{\xstarspxk}{x^{\star}_{S_{k}}}
\newcommand{\xstarspxnotk}{x^{\star}_{-S_{k}}}

\theoremstyle{definition}

\newtheorem{assumption}{Assumption}[section]
\newtheorem{theorem}{Theorem}[section]
\newtheorem{corollary}{Corollary}[section]

\newcommand{\problem}[1]{\textcolor{red}{\textbf{(P{#1})}}}
\newcommand{\solved}[1]{\textcolor{blue}{\textbf{(P{#1})}}}

\title{\textbf{LACToS}: \textbf{L}ocal \textbf{A}nomaly detection via \textbf{C}onditional \textbf{T}ypicality of \textbf{S}uperpixels}

\author{
  Jihoon Hong$^{1}$\thanks{Equal contribution.},
  Alireza Kheirandish$^{1}$\footnotemark[1],
  Rohan Sanda$^{2, 3}$ \& Sara Fridovich-Keil$^{1}$ \\
  $^{1}$Georgia Institute of Technology, Atlanta, GA, USA \\
  $^{2}$Stanford University, Stanford, CA, USA \\
  $^{3}$Cohere \\
  \texttt{\{jhong392,akheirandish3,sfk\}@gatech.edu} \\
  \texttt{rsanda@stanford.edu}
}

\iclrfinalcopy
\begin{document}

\maketitle

\begin{abstract}
Local anomaly detection is the task of identifying spatially confined anomalous or out-of-distribution (OOD) regions within an otherwise in-distribution (ID) image.
Most existing methods are designed for a specific context, with a strong focus on industrial defect detection, where they achieve excellent performance on well-established benchmarks.
However, we observe that these methods can exhibit significant performance drops in other contexts, especially when the inlier distribution is multimodal.
To account for this inlier complexity, we propose \name{}, a principled local OOD detection method based on statistical typicality and made tractable through two innovations: (i) grouping pixels into superpixels and (ii) using a computationally efficient proxy to estimate typicality.
We evaluate \name{} across three diverse applications spanning medical imaging, human faces, and X-ray security screening.
\name{} performs more robustly than prior methods across these benchmarks.
The code is released here: \href{https://github.com/voilalab/LACToS}{https://github.com/voilalab/LACToS}.
\end{abstract}

\section{Introduction}
\label{introduction}

\begin{wrapfigure}{R}{0.6\textwidth}
\centering
\vspace{-5mm}
\includegraphics[width=\linewidth, trim=1.6cm 0cm 0cm 0.1cm, clip]{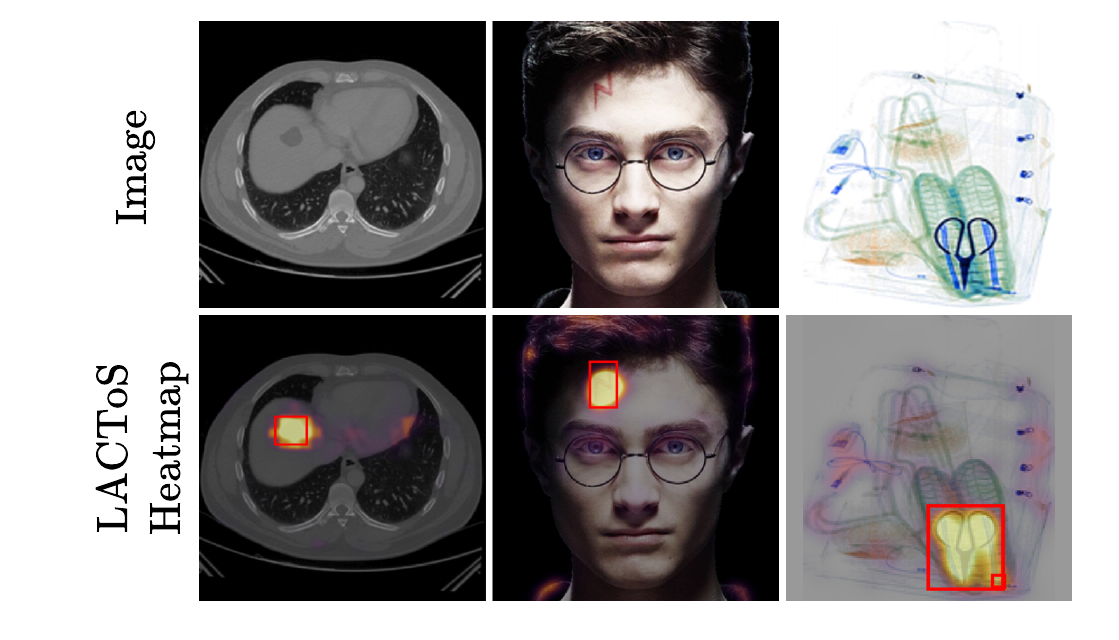}
\vspace{-10mm}
\caption{\label{fig:teaser} Using a diffusion model pretrained on ID images only, \name{} is able to identify locally OOD regions in images across various tasks, including tumors in CT scans, anomalies such as scars on faces, and harmful objects such as scissors in X-ray baggage scans.}
\vspace{-2mm}
\end{wrapfigure}

Local anomaly detection is the task of identifying out-of-distribution (OOD) regions that are spatially confined within an otherwise in-distribution (ID) image.
It is central to applications such as industrial defect inspection, medical image analysis, and security screening, where an anomalous region may occupy only a small fraction of the image yet indicate a defect, tumor, or weapon.
Unlike image-level OOD detection, which assigns a single anomaly score to an entire image, localized pixel-level anomaly detection involves generating a dense heatmap that indicates which pixels or regions of each outlier image are inconsistent with the inlier distribution.
This makes the problem especially challenging: a local OOD detector must be sensitive to subtle local deviations while remaining robust to natural variation in normal image content.

Prior work on localized anomaly detection has largely followed representation-based and reconstruction-based approaches.
Representation-based methods extract local features from normal (inlier) training images, and score test patches using their distance from the learned distribution or their likelihood under it \citep{spade, padim, patchcore}.
Reconstruction-based methods instead reconstruct, inpaint, or denoise the target image under a prior for inlier images, and use the discrepancy between the input and its reconstruction as an anomaly signal \citep{ddad, MDPS}.
Although both approaches have achieved strong empirical performance, especially for industrial defect detection \citep{mvtec}, their anomaly scores are often heuristic: representation-based methods depend on the choice of feature and distance metric, while reconstruction-based methods typically collapse the distribution of inlier completions into a point estimate.
Conformal and risk-control methods can provide statistical guarantees for calibrated segmentation outputs \citep{conformaldefect}, but they generally rely on a large amount of OOD calibration data.
Moreover, we empirically observe that these methods do not always generalize reliably beyond industrial defect detection.

We present \name{}, a statistical method for local OOD detection that is designed to address these limitations of existing methods. Our approach is fundamentally reconstruction-based, but uses the key insight that the distribution of the reconstructions contains valuable information that is lost when reconstruction error is computed using a single sample or an aggregated point estimate.
To leverage this richer distributional information towards improved local anomaly detection, we propose an anomaly metric based on statistical typicality, a measure of how representative a sequence of observations is for a given distribution.
Our use of typicality is inspired by previous typicality-based OOD detection methods \citep{nalisnicktypicality, osadatypicality, barkley2026scoped}.
However, these works are focused on image-level OOD detection where an entire image is classified as OOD, and to the best of our knowledge, using typicality for local anomaly detection has not yet been explored.

We begin with the typicality score of \citet{nalisnicktypicality}, and adapt it to identify local rather than whole-image anomalies (\Cref{subsec:typicality}).
Instead of asking whether an entire image is typical with respect to the unconditional prior distribution, we ask whether each pixel is typical with respect to the \emph{inpainting posterior} conditioned on all other pixels in the image. 
However, this substitution induces three major challenges: (i) failure to detect larger local anomalies that span multiple contiguous pixels, (ii) loss of statistical validity, because typicality testing requires sampling a sequence of i.i.d. observations, and (iii) prohibitive computational cost of performing posterior sampling separately for each pixel.

\name{} tackles these challenges through two innovations: we partition images into superpixels and replace pixel-level typicality with superpixel-level typicality (\Cref{subsec:superpixel_level_typicality}), and we use a single \emph{denoising posterior} as a common proxy for all inpainting posteriors. Our use of superpixels addresses challenges (i) and (ii), because a local OOD feature is more likely to be captured by a superpixel than a single pixel, and the individual pixels in each superpixel can be treated as independent observations to evaluate superpixel-level typicality. Our use of denoising rather than inpainting addresses challenge (iii), by dramatically reducing the number of posterior samples needed to estimate typicality.
Through these innovations, we show that \name{} achieves strong quantitative and qualitative performance across a wide range of local OOD detection tasks including detecting tumors in CT scans, detecting unusual facial accessories in photographs, and detecting weapons in X-ray baggage scans.

To summarize, we propose \name{}, a statistical approach to local OOD detection in images using an inlier diffusion model and typicality testing. Our method (1) 
is grounded in statistical typicality, and (2) is competitive with, and often more precise than, existing methods across diverse local OOD detection tasks with complex inlier distributions.
\section{Preliminaries}
\label{sec:preliminaries}

In \name{}, we treat a diffusion model as an oracle denoiser that samples images consistent with its training distribution and a noisy input image.
We provide a brief introduction to diffusion models and their use in posterior sampling in Appendix \ref{sec:background}. 
Here, we provide context on local OOD detection methods proposed in prior work, some of which inspired our use of statistical typicality in \name{}.

OOD detection is studied at both the \emph{image} and \emph{pixel} (local) levels. 
Image-level OOD detection assigns a single anomaly score to an entire image, with the goal of identifying images that are not drawn from the inlier distribution seen during training.
In contrast, we focus on pixel-level OOD detection, in which the goal is to produce a dense anomaly heatmap to localize anomalous regions within each image.
Formally, given an image $\xstar$, the goal is to estimate the ground truth anomaly mask $\mstarpixel$.
Previous works \citep{ddad,anoddpm,padim} define an anomaly score $f(\xstar,j)$ that captures how likely each pixel $j\in\J$ is to be OOD, where $\J$ is an index set of size $W\times H$ over all pixels.
Then, they compute an anomaly heatmap $A$ with elements $A_j=f(\xstar,j)$ for each pixel, fix some threshold $\tau$, and set ${\mhatpixel}=\mathbbm{1}\{A\geq\tau\}$.
Many of these works, especially those using diffusion models, are reconstruction-based: they generate an in-distribution reconstruction of the input using posterior sampling and use discrepancies such as $\ell_1$, $\ell_2$, SSIM \citep{ssim}, LPIPS \citep{lpips}, or feature-space distances as anomaly scores \citep{RIAD, MDPS, ddad, InTra}. 
Although effective, these methods often collapse the distribution of plausible reconstructions into a first-order statistic, such as the mean reconstruction error, which does not capture the full reconstruction distribution and can reduce robustness to complex inlier distributions.

Meanwhile, distribution-aware statistics have been explored across several image-level OOD detection methods. 
For example, EigenScore \citep{eigenscore} uses the eigenvalue spectrum of the posterior covariance, and DDPM-OOD \citep{ddpmood} aggregates standardized reconstruction statistics.
There are also typicality-based methods that compare the average negative log-likelihood of the samples with the entropy of the inlier distribution \citep{nalisnicktypicality, osadatypicality, barkley2026scoped}. 
Specifically, let 
$(x^{(1)},\ldots,x^{(N)})$ be independent draws from the inlier distribution $p$. 
The $(\epsilon,N)$-typical set is defined as
\begin{equation}
\label{eq:typical_set}
    \mathcal{T}^{p}_{\epsilon,N}
    =
    \left\{
    (x^{(1)},\ldots,x^{(N)})
    \;\middle|\;
    \left|
    -\frac{1}{N}\log p(x^{(1)},\ldots,x^{(N)})
    -
    \entropy(p)
    \right|
    \leq \epsilon
    \right\},
\end{equation}
where $\entropy(p)=\mathbb{E}_{x\sim p}[-\log p(x)]$
is the entropy of $p$. 
By the Asymptotic Equipartition Property \citep{coverthomas},
\begin{equation}
\label{eq:aep}
    -\frac{1}{N}\log p(X^{(1)},\ldots,X^{(N)})
    \xrightarrow[N\rightarrow\infty]{}
    \entropy(p).
\end{equation}

For image-level OOD detection, \citet{nalisnicktypicality} estimate the deviation from typicality as
\begin{equation}
\label{eq:typicality}
    \hat{\epsilon}(x^{(1)},\ldots,x^{(N)})
    :=
    \left|
    -\frac{1}{N}\sum_{n=1}^{N}\log p(x^{(n)})
    -
    \entropy(p)
    \right|,
\end{equation}
and classify a set of samples as OOD when $\hat{\epsilon}$ exceeds a threshold. 
Typicality has been shown to be a more effective image-level OOD detection metric than likelihood, especially in high dimensions \citep{nalisnicktypicality, osadatypicality}. 
\section{Motivation}
\label{sec:motivation}

\begin{wrapfigure}{R}{0.55\textwidth}
\centering
\vspace{-5mm}
\includegraphics[width=\linewidth]{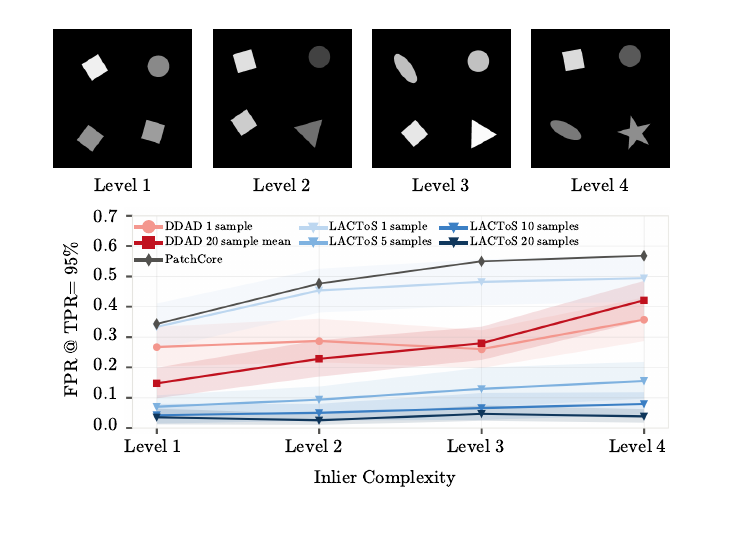}
\vspace{-12mm}
\caption{\label{fig:toy_experiment_fpr} \textbf{Multimodality Toy Experiment.} (Top) Images from the inlier distributions across different complexity levels. The top-right corner always contains a circle, while other corners can contain different shapes, with a wider variety of shapes at higher levels. (Bottom) False positive rates at true positive rate of 95$\%$, evaluated on images that contain anomalous squares in the top-right corner.}
\vspace{-2mm}
\end{wrapfigure}

The local OOD detection problem has been extensively studied in the context of industrial defect detection.
Prior works such as PatchCore \citep{patchcore}, EfficientAD \citep{efficientad}, GLASS \citep{glass}, and diffusion-based DDAD \citep{ddad} achieve excellent performance on widely used benchmarks such as MVTec AD \citep{mvtec} and VisA \citep{visa}.
However, we observe that these methods do not extend robustly to other OOD detection tasks, such as tumor detection, facial anomaly detection, and security screening using X-ray images, as shown in \Cref{tab:main_all_pooled}.
One potential cause of this brittleness is the complexity of the inlier distribution.
Industrial defect-detection datasets often have very structured inlier distributions, as products without defects are visually near-homogeneous, whereas CT scans, images of human faces, and X-ray baggage scans are far more heterogeneous even in the absence of anomalies.
Moreover, defects in products are context-agnostic: a defective feature, such as a crack or fracture, is anomalous regardless of where in the image or product it occurs.
In contrast, a dark region in a CT scan may be innocuous normal anatomy in one location, or indicative of a tumor in another.

Specifically, when the conditional distribution over a local region given the rest of the image is multimodal, methods that rely on a first-order statistic may erroneously identify samples from a different inlier mode as anomalous. 
To illustrate this phenomenon through a toy experiment, we construct four inlier distributions with increasing complexity levels, and train a diffusion model on each inlier distribution. Each image contains four shapes, one per quadrant. 
In the upper-right quadrant, every inlier image contains a circle, which allows us to define all other shapes placed in that quadrant as local anomalies. 
The other three quadrants can contain a wider variety of shapes, as shown in the top row of \Cref{fig:toy_experiment_fpr}, with the candidate shapes at each level occurring with equal probability. From level 1 to level 4, the set of shapes that can be placed in these three quadrants progressively grows, which increases the number of modes of the inlier distribution. 
We then introduce local anomalies in test images by replacing the top-right circle with a square, and evaluate two baseline methods, DDAD \citep{ddad} and PatchCore \citep{patchcore}, on how well they can correctly identify only the square as an anomaly. 
The bottom row of \Cref{fig:toy_experiment_fpr} demonstrates that at a per-image true positive rate (TPR) of 95$\%$, the false positive rate (FPR) increases with inlier multimodality for both baseline methods. This trend holds across the TPR values shown in \Cref{fig:fpr_at_tpr}.
For more information on the experimental setup, please refer to \Cref{sec:toy_experiment_1}.

\section{Method}
\label{sec:methods}

To tackle this problem, we propose \name{}, a typicality-based pixel-level OOD detection method using diffusion models, which is much more robust to inlier multimodality.
In \Cref{subsec:typicality}, we first describe an idealized extension of an existing typicality-based metric for image-level OOD detection to pixel-level anomaly localization, and several issues that limit its direct application. Then, in \Cref{subsec:superpixel_level_typicality} and \Cref{subsec:denoising} we introduce two key techniques, superpixels and a denoising proxy, that help overcome these issues and allow for an effective and tractable approximation to pixel-level typicality. In \Cref{subsec:feature}, we additionally incorporate extracted pixel-level features which can provide more informative statistical context than color alone.

Throughout, we use $\xstar$ to denote the test image and subscripts to identify specific pixels or superpixels. Similarly, we use $\x$ (again with appropriate subscripts) to denote the corresponding random variable under the prior or posterior, with Monte Carlo realizations indexed by superscripts.

\subsection{Typicality-based Anomaly Score}
\label{subsec:typicality}

A natural extension of the existing typicality-based OOD detection metric in \Cref{eq:typicality} from images to pixels is conditional typicality.
It is defined as
\begin{equation}
\label{eq:typicality_conditional}
    f^{\mathrm{inp}}(\xstar,j):=
    \big|-\log{\pj{\xstarj}-}\entropy(\pj{\xj})\big|,
    \qquad
    j\in\J,
\end{equation}

where the superscript $\mathrm{inp}$ stands for inpainting, $\xstarj$ and $\xstarnotj$ are the values of pixel $j$ and all remaining pixels in the test image and all other pixels, respectively,
and $\pj{\xj}$ is the distribution of prior-consistent values for pixel $j$ when all other pixel values are fixed at their observed values $\xstarnotj$.
Intuitively, we define a pixel in $\xstar$ to be OOD if it does not belong to the typical set with respect to this \emph{conditional} data distribution $\pj{\xj}$, given all other pixel values.

Unfortunately, this direct extension of typicality to conditional distributions in \Cref{eq:typicality_conditional} is not a practical score for anomaly detection for the following reasons.
\problem{1} Due to the spatial locality of pixel intensities and the smoothness of natural images, the value of a pixel given all others is often strongly constrained by nearby pixels.
With conditional distributions that are concentrated on a tiny set determined by the neighboring pixels, the typicality score is unlikely to capture OOD features that span multiple adjacent pixels.
\problem{2} In addition, a single pixel value is a sequence of length 1, which is insufficient for the theoretical guarantees of typical sets (\Cref{eq:aep}) to hold.
Indeed, the typicality test of \citet{nalisnicktypicality} has been shown to exhibit poor empirical performance on single samples \citep{typicalityfail}.
\problem{3} It is also computationally demanding to compute the anomaly heatmap $A$ by approximating \Cref{eq:typicality_conditional}, as $f^{\mathrm{inp}}(\xstar,j)$ has to be evaluated for every pixel $j\in\J$. Using $N$ posterior samples for a Monte Carlo estimate $\hat{p}(\xj|\xstarnotj)$ of $\pj{\xj}$ for each pixel
would require a total of $N\times|\J|$ posterior samples, which is computationally prohibitive at practical image resolution $H,W$.
\problem{4} Finally, raw RGB values often do not capture sufficient information about local image contents, including semantic information, which can lead to subpar performance in anomaly detection tasks.

\subsection{Superpixel-level Typicality}
\label{subsec:superpixel_level_typicality}

\begin{wrapfigure}{R}{0.3\textwidth}
\centering
\vspace{-8mm}
\includegraphics[width=\linewidth]{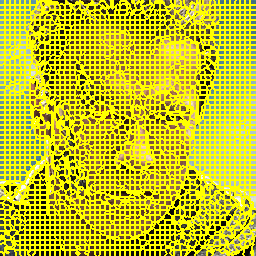}
\vspace{-6mm}
\caption{\label{fig:superpixels}A fine-grained superpixel decomposition of an example image.}
\vspace{-3mm}
\end{wrapfigure}
Our first key insight is that superpixels \citep{superpixel_segmentation, slic, superpixel_survey}, groups of spatially contiguous pixels that share similar color,
can help us address \problem{1} and \problem{2}, by grouping OOD pixels together into small groups that are easier to detect and can be treated as a sequence to estimate typicality. The latter requires Assumption~\ref{assumption:conditional_independence}, which captures the intuition that, within each superpixel, individual pixel values are independent and identically distributed. An example of superpixels is visualized in \Cref{fig:superpixels}.

Concretely, we define $\spxs=\{\spx{k}\subset\J\mid k\in\K\}$ to be a fine-grained superpixel partition of the pixels in $\J$, which can be obtained using off-the-shelf superpixel algorithms \citep{slic, seeds, ers, lsc}. Here $k$ indexes superpixels and $S_k$ is the set of pixels in superpixel $k$.
Let $\xspxk$ denote image pixels in superpixel $k$, and $\xspxnotk$ denote all other pixels. 

\begin{assumption}
\label{assumption:conditional_independence}[Conditionally i.i.d. pixel values within each superpixel]
Given $\xstar$ and $\spxs$, there exists a single-pixel distribution $\qk$ associated with each superpixel $k \in \K$ such that
    \begin{equation*}
        \quad
        \xj\mid\xstarspxnotk\overset{iid}{\sim}\qk
        \quad \forall j\in\spxk.
    \end{equation*}
\end{assumption}
Under Assumption~\ref{assumption:conditional_independence}, we define the \emph{superpixel-level typicality score} $g^{\mathrm{inp}}$ as
\begin{equation}
\label{eq:typicality_superpixel}
    g^{\mathrm{inp}}(\xstar,k)
    =
    \frac{1}{|\spxk|}
    \Big|
        -\log{\pk{\xstarspxk}}
        -\entropy(\pk{\xspxk})
    \Big|
    =
    \Bigg|
        -\frac{1}{|\spxk|}
        \sum_{j\in\spxk}\log{\qk(\xstarj)}
        -\entropy(\qk)
    \Bigg|.
\end{equation}
where $\pk{\xspxk}$ is the conditional distribution of the superpixel given all other pixel values, and the last equality holds due to the conditional independence assumption (Assumption~\ref{assumption:conditional_independence}) on the pixels in $\spxk$.
We can then compute the pixel-level anomaly heatmap $A$ by estimating the typicality score in \Cref{eq:typicality_superpixel} using Monte Carlo sampling. 
For each superpixel $k$, this involves drawing $M$ samples $x^{(1)},\ldots,x^{(M)}$ from the inpainting distribution 
$\pinpk:=p(\x\mid\xstarspxnotk)$,
extracting pixels $\{x^{(m)}_j\}_{j\in \spxk,m\in[M]}$, fitting a histogram $\hat{q}_k^{\mathrm{inp}}$ to estimate $\qk$, and computing the estimate $\hat{g}^{\mathrm{inp}}(\xstar,k)$. Finally, for pixel-level scores it is natural to duplicate the same score over all pixels $j\in\spxk$, i.e. $\hat{f}^{\mathrm{inp}}(\xstar,j)=\hat{g}^{\mathrm{inp}}(\xstar,k)$ for all $j\in\spxk$, and apply Gaussian smoothing to smooth out the hard boundaries between superpixels.

Estimating the superpixel-level typicality score $g^{\mathrm{inp}}$ rather than the pixel-level typicality score $f^{\mathrm{inp}}$ helps address these problems. 
\solved{1} By inpainting a group of similar pixels instead of a single pixel, $g^{\mathrm{inp}}$ can better identify local OOD features that span multiple pixels than $f^{\mathrm{inp}}$.
\solved{2} Under Assumption~\ref{assumption:conditional_independence}, we are given a sequence of $|\spxk|$ independent observations $\{\xstarj\}_{j\in\spxk}$ for $\qk$, which facilitates computing a typicality score for each superpixel.
\solved{3} In terms of computational cost, the number of anomaly scores to evaluate is reduced from the  number of pixels $|\J|$ to the number of superpixels $|\K|$.
We can also use a much smaller number of samples $\M\ll N$ to compute the histogram $\hat{q}_k$ for $\qk$, since each posterior sample from $\pinpk$ translates to $|\spxk|$ samples from $\qk$.

\newcommand{\plotcrop}[1]{%
\adjincludegraphics[trim={{0.0\width} {0.06\height} {0.0\width} {0.065\height}}, clip, width=\linewidth]{#1}%
}

\begin{figure}[t]
  \centering
  \plotcrop{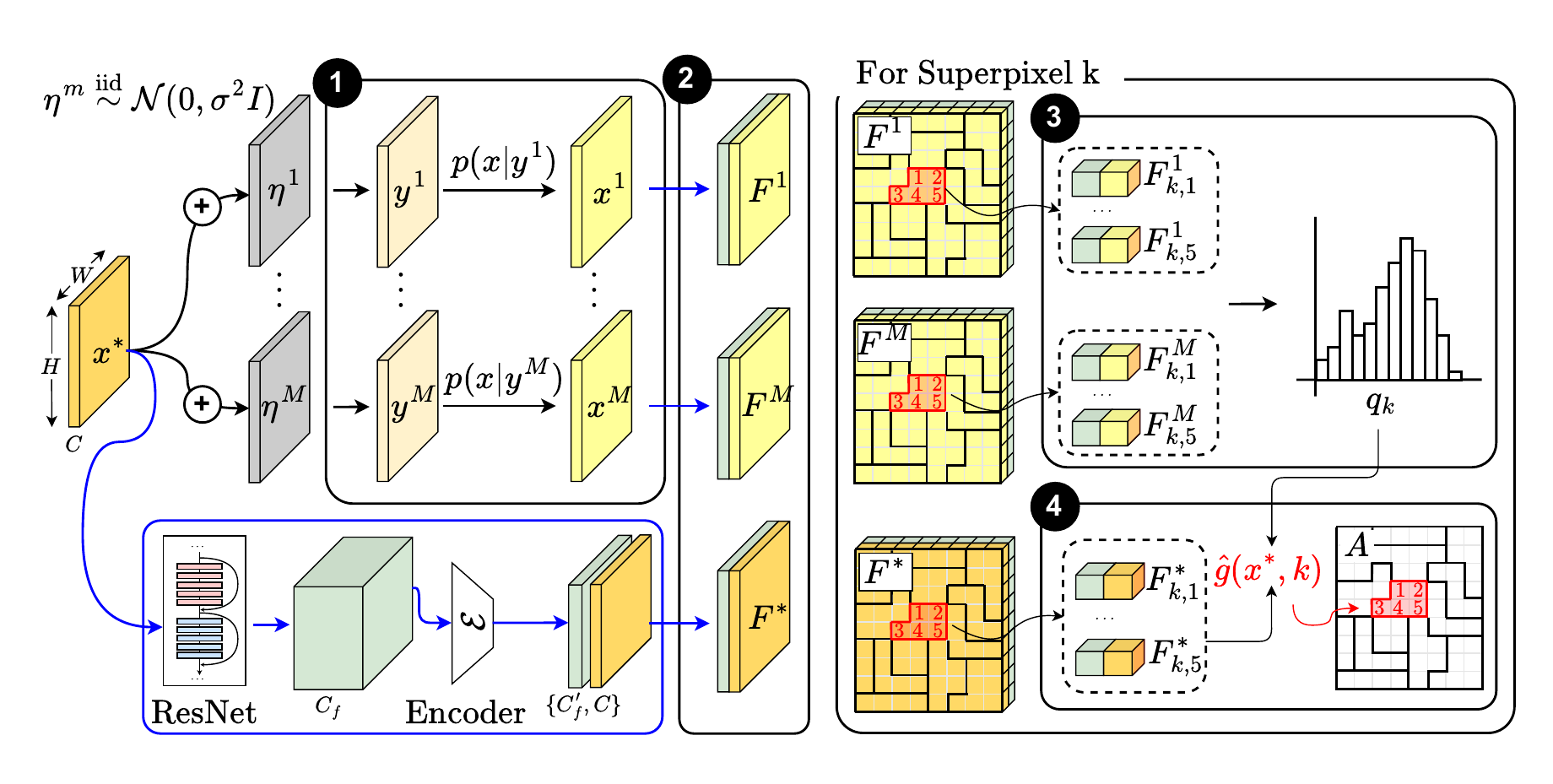}
  \vspace{-5mm}
  \caption{\textbf{\name{} Overview.} \textbf{(1) Posterior Sampling.} We sample noise independently from an isotropic Gaussian distribution and add it to the test image $\xstar$ to create noised versions of $\xstar$ denoted by $y^{m}$. Then, using diffusion-based posterior sampling methods, we sample reconstructions $x^{m}\sim p(x|y^{m})$. \textbf{(2) Feature Extraction.} For $\xstar$ and each sample $x^m$, we pass the image through a pretrained ResNet model, extract activations, and compress them using an encoder to create a feature vector of dimension $C_f'+C$ for each pixel, where $C$ is the original channel dimension. \textbf{(3) Histogram Fitting}. For a given superpixel $k$, we gather the features of the reconstructed pixels in the superpixel, and construct $\hat{q}_k^{\mathrm{den}}$, a $C_f'+C$-dimensional histogram. \textbf{(4) Typicality Computation.} Finally, using the fitted histogram, we compute the typicality of the pixel features in $\xstar$ and report the anomaly heatmap $A$.}
  \vspace{-3mm}
  \label{fig:methods}
\end{figure}

\subsection{Denoising Posterior}
\label{subsec:denoising}

Although superpixel-level typicality reduces the required number of posterior samples from $|\J|N$ to $|\K|\M$, the cost remains prohitive for fine superpixel partitions.
For instance, \Cref{fig:superpixels} shows a 256$\times$256 image partitioned into 2,600 superpixels (approximately 25 pixels per superpixel on average), which would require a total of at least 2600 posterior samples using a diffusion model.
This cost could be reduced by partitioning the image into fewer, larger superpixels, but these may struggle to localize small anomalies and fail to satisfy Assumption~\ref{assumption:conditional_independence}.

We overcome this problem \problem{3} by estimating $\qk$ for all $k\in\K$ simultaneously using a small number of posterior samples.
Instead of using different samples from each of the \emph{inpainting conditional distributions} $\{\pinpk\}_{k\in\K}$, we use the same $\M$ samples from the \emph{denoising conditional distribution}, defined as 
\begin{equation}
    \pden(\x):=\E_{\eta\sim\mathcal{N}(0,\sigma^2I)}\big[p(\x\mid \xstar+\eta)\big].
\end{equation}
For each $\spxk$, we pool the pixel values across these samples to form $\{\xmj\}_{j\in\spxk,m\in[M]}$, which we use to fit a histogram for $\qk$, and compute the anomaly score. 

Intuitively, $\pden$ serves as a shared relaxation of the superpixel-specific inpainting posteriors $\{\pinpk\}_{k\in\K}$.
In each $\pinpk$, the complement $\xstarspxnotk$ is fully observed while $\xstarspxk$ is left unobserved.
For the denoising posterior $\pden$, there is a tension between information on $\xstarspxnotk$ and lack of information on $\xstarspxk$, controlled by the noise level $\sigma$. With increasing $\sigma$, we lose more information about $\xstarspxk$, which is desirable, but we also lose information about the context $\xstarspxnotk$.
The following theorem informally bounds the error introduced by replacing the inpainting posterior with the denoising posterior. Appendix \cref{app:theory} states and proves the formal result and gives bounds without Assumption~\ref{assumption:conditional_independence}.

\begin{theorem}
\label{thm:main}
\textbf{(Informal).} Let $g^{\mathrm{inp}}(\xstar,k)$ and $g_\sigma^{\mathrm{den}}(\xstar,k)$ be the anomaly scores obtained using the ideal inpainting distribution and the proxy denoising distribution with a noise distribution of $\mathcal{N}(0,\sigma^2I)$, respectively. Under appropriate assumptions, the error can be bounded as
    \begin{equation*}
        \left|g^{\mathrm{inp}}(\xstar,k) - g_\sigma^{\mathrm{den}}(\xstar,k)\right|\leq \alpha_k\sigma + \beta_k\sigma^{-\gamma_k}
    \end{equation*}
for some positive constants $\alpha_k$, $\beta_k$, and $\gamma_k$.
\end{theorem}

\subsection{Feature Augmentation}
\label{subsec:feature}

One final problem remains \problem{4}: RGB values alone may not distinguish local ID and OOD features.
To address this issue, \solved{4} we follow the convention of previous works \citep{ddad, patchcore, padim} and augment RGB values with additional features.
Specifically, we extract the activations of the last three layers of a pretrained ResNet \citep{resnet} $\phi$, as later layers of convolutional neural networks are known to encode higher-level visual structure and have been widely used as transferable representations for downstream vision tasks \citep{cnnfeatures,cnnvisualization}.
The high activation dimension, denoted as $C_f$, can increase histogram estimation error when few samples are available.
We therefore use an autoencoder trained on each ID dataset to compress the features to dimension $C'_{f}(\ll C_f)$; in our experiments we use $C'_f=3$.

\paragraph{Algorithm Summary.} Our complete algorithm, illustrated in \Cref{fig:methods}, can be summarized as follows.
First, we generate $\M$ samples from the denoising conditional distribution $\pden$ through posterior sampling using a diffusion model.
Second, each of the samples is passed through $\phi$ to extract features, which are then compressed using the trained encoder $\mathcal{E}$.
Third, the pixel values and the compressed features are concatenated, resulting in 
\begin{equation*}
    u_{j}^m:=(x_{j}^m,\mathcal{E}(\phi(x^m)_j))\in\R^{C+C_f'}
\end{equation*}
for each pixel $j$.
Fourth, the augmented vectors $u_{j}^m$ that belong to each superpixel $\spxk$ are collected to form $\{u_j^{m}\}_{j\in\spxk,m\in[\M]}$ and used to construct a histogram $\hat{q}_k$, an estimate of the $C+C_f'$-dimensional distribution $\qk$.
The fourth step is repeated over all superpixels, after which we can compute the anomaly map by computing the typicality of $\xstar$ for each superpixel (\Cref{eq:typicality_superpixel}), and finally apply spatial Gaussian smoothing to refine pixel-level anomaly scores across superpixel boundaries.
\section{Evaluation}

\paragraph{Datasets.}
We evaluate \name{} on three local anomaly detection datasets spanning medical imaging, human faces, and X-ray security screening (\Cref{tab:dataset_summary}). In each case, the diffusion model in \name{} is trained only on ID images, and all methods are evaluated on a held-out test set. Each test set includes an equal number of ID images and images containing local OOD features.

\textbf{\chaos{}} \citep{chaos} provides abdominal CT scans of healthy patients. We train on 3{,}070 inlier slices and construct a test set containing 100 healthy test slices and 100 locally anomalous slices acquired by adding synthetic tumor to each.
Following common practice in CT image analysis \citep{bodymask1, bodymask2, bodymask3}, we apply the provided body mask at evaluation to exclude background pixels.
\textbf{\ffhq{}} \citep{ffhq} is a dataset of high-quality face images, which we treat as the inlier distribution. We train on a random subset of 5{,}000 images, and manually curate an OOD test set of 50 photographs of movie and television characters with prosthetics or atypical facial features (e.g., scars, metal implants), verified not to overlap with FFHQ. The ID test set contains 50 randomly selected held-out FFHQ images.
\textbf{\xray{}} \citep{dvxray} is a dataset of X-ray luggage scans, with an ID set of weapon-free bags and an OOD set containing weapons. We train on 5{,}000 inlier images and restrict the OOD test set to the \emph{scissors} category, where targets are small and spatially confined.
For randomly sampled scissor images, we used SAM~3~\citep{sam3} to convert the dataset's bounding-box annotations into segmentation masks, manually discarding images with visibly inaccurate masks for a final OOD test set of 72 images. We construct an ID test set that also contains 72 images, none of which contains a weapon.

\paragraph{Baselines \& Models.}

We compare \name{} against six strong baselines developed for industrial defect detection, spanning representation- and reconstruction-based methods.
\textbf{PatchCore} \citep{patchcore} builds a coreset-subsampled memory bank of ID patch features from a frozen ImageNet-pretrained backbone and scores test patches by nearest-neighbor distance.
\textbf{SimpleNet} \citep{simplenet} learns a feature adapter and a discriminator that separates adapted ID features from synthetic anomalies generated by injecting Gaussian noise in feature space.
\textbf{SuperSimpleNet} \citep{supersimplenet} extends SimpleNet with architectural and training improvements for fast, reliable surface defect detection.
\textbf{EfficientAD} \citep{efficientad} combines a lightweight student--teacher model with an autoencoder to detect local structural anomalies and violations of global logical constraints through feature discrepancies.
\textbf{GLASS} \citep{glass} trains a discriminator to separate ID features from synthetic anomalies generated using gradient ascent with truncated projection in feature space and local texture overlays in image space.
\textbf{DDAD} \citep{ddad} conditions a denoising diffusion model on the test image to produce an ID reconstruction and scores anomalies by feature-space discrepancy.

\name{} requires an unconditional diffusion model and a posterior sampler. For all benchmarks, we use the guided-diffusion U-Net of \citet{guideddiffusion} with diffusion posterior sampling (DPS) \citep{dps} at inference time. Full experimental details are provided in \Cref{sec:experiment_details} and in our code, which will be released upon publication. 

\begin{table}[t]
  \caption{
  Pooled quantitative performance in local anomaly detection.
  AUROC and AP are computed over all pixels in each dataset jointly.
  AUPRO is computed over the dataset using per-region overlap for
  $\mathrm{FPR} \leq 0.3$.
  Higher is better for all metrics.
  The best result is shown in \textbf{bold} and the second-best is
  \underline{underlined}.
  }
  \label{tab:main_all_pooled}
  \centering
  \small
  \setlength{\tabcolsep}{2.5pt}
  \resizebox{\linewidth}{!}{%
    \begin{tabular}{llccccccc}
      \toprule
      Dataset & Metric
      & PatchCore & SimpleNet & SuperSimpleNet & EfficientAD
      & GLASS & DDAD & \name{} (Ours) \\
      \midrule

      \multirow{3}{*}{\chaos{}}
      & AUROC & 0.9297 & 0.7464 & 0.3786 & 0.6664 & 0.8318 & \textbf{0.9707} & \underline{0.9620} \\
      & AP    & 0.0590 & 0.0193 & 0.0051 & 0.0105 & 0.0410 & \underline{0.3159} & \textbf{0.4318} \\
      & AUPRO & 0.6811 & 0.4191 & 0.0332 & 0.2513 & 0.5530 & \underline{0.8676} & \textbf{0.8764} \\
      \midrule
      \multirow{3}{*}{\ffhq{}}
      & AUROC & \underline{0.8491} & 0.7162 & 0.7647 & 0.6200 & 0.6707 & 0.7469 & \textbf{0.8515} \\
      & AP    & \underline{0.0798} & 0.0537 & 0.0772 & 0.0405 & 0.0345 & 0.0563 & \textbf{0.1177} \\
      & AUPRO & \underline{0.4094} & 0.3174 & 0.4047 & 0.2365 & 0.2440 & 0.3083 & \textbf{0.5320} \\
      \midrule
      \multirow{3}{*}{\xray{}}
      & AUROC & \textbf{0.9640} & 0.9234 & 0.8688 & 0.9184 & 0.8944 & 0.8567 & \underline{0.9607} \\
      & AP    & \underline{0.0545} & 0.0254 & 0.0165 & 0.0235 & 0.0155 & 0.0118 & \textbf{0.0568} \\
      & AUPRO & \underline{0.8141} & 0.6577 & 0.5351 & 0.6592 & 0.6062 & 0.5183 & \textbf{0.8179} \\

      \bottomrule
    \end{tabular}%
  }
\end{table}

\begin{figure}[t]
    \centering
    \includegraphics[width=1.0\textwidth, trim=1cm 0cm 0cm 1cm, clip]{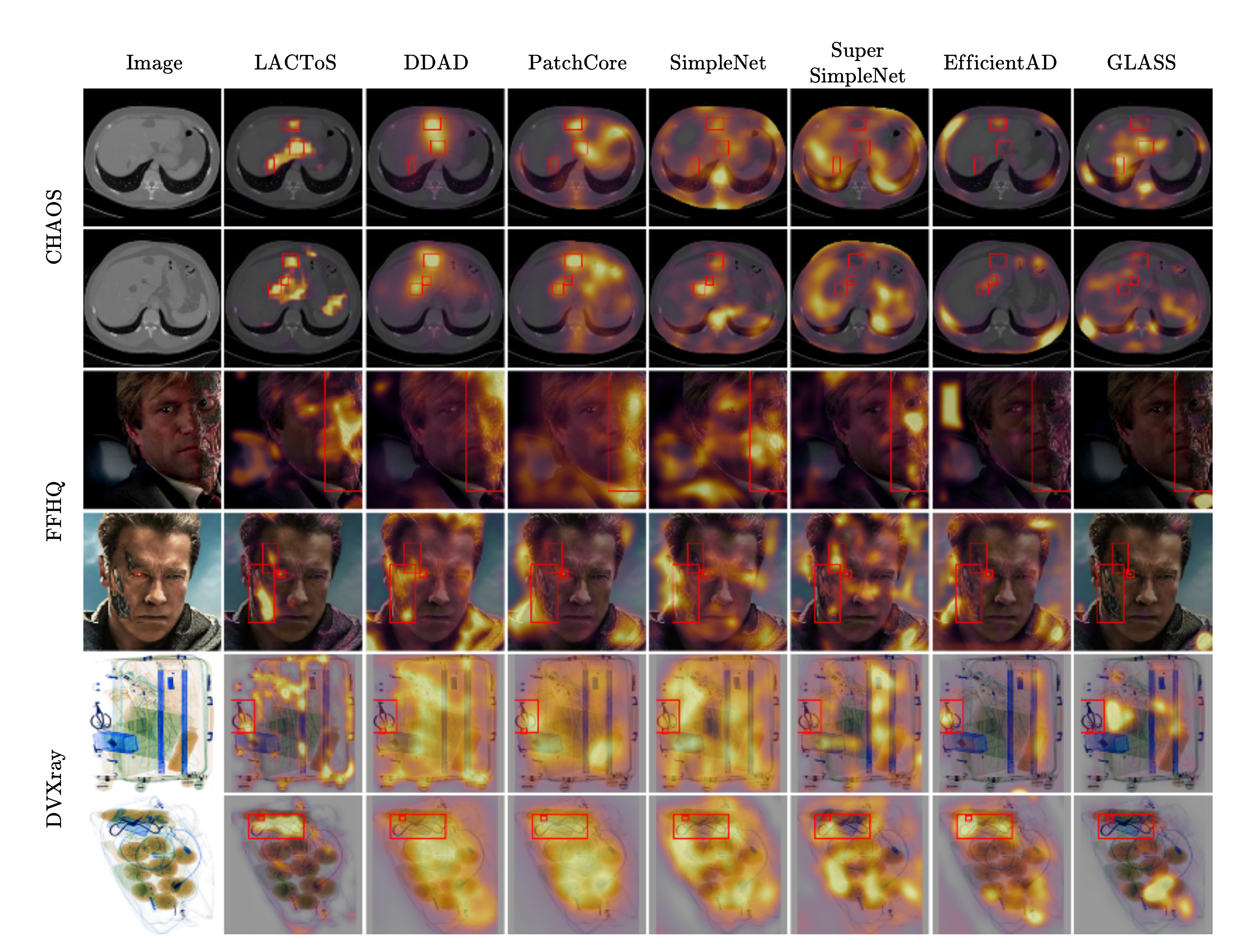}
    \vspace{-8mm}
    \caption{Anomaly heatmaps generated by each method on example images from each test dataset. Rows 1-2 (\chaos{}): detect liver tumors in CT scans. Rows 3-4 (\ffhq{}): detect anomalous facial features. Rows 5-6 (\xray{}): detect weapons (scissors) in X-ray baggage scans.}
    \label{fig:results}
    \vspace{-5mm}
\end{figure}

\paragraph{Results.}

Each method produces a heatmap of anomaly scores per test image. Following standard practice, we pool pixel-level anomaly scores across all test images and evaluate three complementary pixel-level detection metrics. 
\textbf{AUROC} captures the relative ordering of heatmap values against the ground-truth anomaly mask, but can reward low-contrast heatmaps that flag many ID pixels as nearly as anomalous as the true OOD pixels. 
\textbf{AP} summarizes precision across recall levels, weighting each precision value by the corresponding increase in recall.
\textbf{AUPRO} is the normalized area under the per-region overlap curve up to a false positive rate of $0.3$, where per-region overlap is the fraction of correctly detected pixels within each connected ground-truth anomalous region, averaged equally across regions. When different anomalies have different sizes, AUPRO weights all anomalies equally rather than overweighting large anomalies.

\Cref{tab:main_all_pooled} reports these metrics, which indicate that \name{} is competitive with baselines across all tasks, though the best method differs across metrics.
The example heatmaps in \Cref{fig:results} also suggest that \name{} tends to be more precise: it correctly marks OOD pixels but rarely flags ID pixels, at the cost of occasionally missing part or all of the OOD region.
The baselines instead produce blurrier heatmaps that recover the OOD region but also flag many ID pixels.

\paragraph{Ablation Studies.}
To verify the robustness of \name{}, we conduct ablation studies on several hyper-parameters, which we describe in detail in \Cref{sec:ablations}.
These include the number of posterior samples $M$ and the average size of each superpixel.
We observe that the performance of \name{} is mostly stable across different settings, demonstrating its robustness to hyperparameter choices.
In addition, we perform an ablation on the use of autoencoder-based features, which shows that the features can substantially improve performance compared to typicality based solely on RGB color, though this improvement is dataset-dependent.

\section{Discussion}

We present \name{}, a framework for general-purpose local anomaly detection.
We leverage a pretrained diffusion model to represent the inlier distribution, and combine it with a statistical approach that scores image regions by their typicality (or rather, atypicality) under the conditional distribution induced by their surrounding context.
Empirically, \name{} shows strong performance across diverse local OOD detection tasks ranging from medical imaging to security screening.

We note that since \name{} is more robust to inlier multimodality, it is also less likely to flag features from minority inlier modes as anomalous, especially when using a larger $M$.
We demonstrate this behavior through a toy experiment, and refer interested readers to \Cref{sec:toy_experiment_2}.
It is possible that some features annotated as anomalies in the benchmarks we evaluate may be similar to features present with low frequency in the inlier (training) datasets, causing \name{} to classify them as inliers and thus generate false negative predictions according to the dataset anomaly labels.
We therefore advise practitioners that, in order to maximize the accuracy of \name{}, it is important to ensure that the training set does not contain the anomalies one desires to detect.

\paragraph{Limitations and future work.}
Since \name{} targets general-purpose local anomaly detection, specialized approaches can outperform \name{} for specific highly structured datasets, such as those common in industrial defect detection. Please refer to \cref{tab:app_mvtec} for evaluation on MVTec AD \cite{mvtec}.
\name{} inherits the computational cost of diffusion posterior sampling, which is significantly higher per-image than feedforward representation-based methods such as PatchCore.
\name{} also relies on Assumption~\ref{assumption:conditional_independence}, which may be violated when superpixels are large or when the anomalous region spans multiple superpixels with heterogeneous pixel statistics.
\name{} requires several hyperparameters, such as the noise variance in the denoising posterior sampling, the quantization granularity and smoothing bandwidth in the histogram estimation, the number of posterior samples, the granularity of superpixels, and the spatial smoothing variance of the final heatmap. While we provide ablation studies showing \name{} is fairly robust to these hyperparameters, future work may find mechanisms to remove or automatically set optimal parameter values across tasks. Future work may also replace our histogram estimation with a learned or parametric density model, or extend our \name{} framework to video and 3D data types where spatiotemporal or volumetric superpixels could provide longer sequences for improved typicality estimation. 

\section{Aknowledgement}

This work was supported in part by the HIVES program at Georgia Tech Research Institute (GTRI).
{
    \small
    \bibliographystyle{plainnat}
    \bibliography{main}
}
\clearpage

\appendix

\section{Background}
\label{sec:background}
\subsection{Diffusion Models}
\label{subsec:diffusion_models}

Diffusion models are generative models that have gained enormous popularity for their high sample quality \cite{ddpm, guideddiffusion}.
To sample from a desired data distribution $p$, these models first define a noising process (also called the forward process) as a stochastic differential equation (SDE) \citep{scoresde} of the form
\begin{equation}
    dx=f(x,t)dt+g(x,t)dw,\qquad t=0\rightarrow 1,
\end{equation}
where $w$ is a standard Wiener process (Brownian motion) and $f$, $g$ are the drift and diffusion coefficients, respectively, which are determined by a chosen noise schedule.
If we denote the marginal distribution of $x$ at time $t$ as $p_t$, this process transforms clean data samples $x\sim  p_0:=p$ across time until the distribution becomes pure noise at $t=1$, i.e. $p_1=\mathcal{N}(0,I)$.
Once the forward process is defined, there exists, under regularity conditions, a corresponding backward SDE with drift $\tilde{f}$ and diffusion $\tilde{g}$ \citep{reversesde, scoresde}:
\begin{equation}
    dx=\tilde{f}(x,t)dt+\tilde{g}(x,t)d\bar{w},\qquad t=1\rightarrow 0,
\end{equation}
where $\bar{w}$ is also a Wiener process but with $t$ indexed backwards.
This backward SDE shares the same marginal distributions $p_t$ as the marginal distributions induced by the forward process for all $t\in[0,1]$.
Since the drift coefficient $\tilde{f}$ of the backward SDE has the form
\begin{equation}
    \tilde{f}(x,t)=f(x,t)-g(x,t)^2\nabla_x\log{p_t(x)},
\end{equation}
a diffusion model is optimized on either a score-matching loss to predict $\nabla_x\log{p_t(x)}$ \cite{scorematching, scoresde} or on a flow-matching loss to predict $\tilde{f}(x,t)$ directly \cite{flowmatching}, and used to iteratively sample from $p_t$ through discretization of the backward SDE.

\subsection{Posterior Sampling using Diffusion Models}
\label{subsec:posterior_sampling}
To condition generation on inference-time measurements using a pretrained diffusion model, prior works modify the reverse sampling process to approximately sample from the posterior distribution capturing simultaneous consistency with the prior encoded by the diffusion model and the likelihood induced by the measurement model \cite{scoreinverse,ddrm,dps,ddnm}.

These methods decompose the conditional score using Bayes' rule as
\begin{equation}
    \nabla_x\log{p_t(x|y)}=\nabla_x\log{p_t(x)}+\nabla_x\log{p_t(y|x)},
\end{equation}
where $\nabla_x\log{p_t(x)}$ is the unconditional score, which is provided by the pretrained diffusion model, and $\nabla_x\log{p_t(y|x)}$ is the likelihood score, which is typically approximated to promote consistency with $y$ while maintaining computational tractability.
These works have shown strong empirical performance on various imaging inverse problems such as inpainting, denoising, deblurring, super-resolution, and tomography; in all of these tasks the condition is given in the form of measurements $y=\mathcal{A}(x)+\eta$ for linear or non-linear observation model $\mathcal{A}$ and noise $\eta$.

\clearpage
\section{Toy Experiment 1}
\label{sec:toy_experiment_1}

Here we provide further details and results for the toy experiment conducted in \Cref{sec:motivation}. For each level of inlier complexity, we train a dedicated diffusion model on a set of inlier images that always contain a circle at the upper-right corner. The other 3 corners of the inlier images contain a shape randomly selected from the set of candidate shapes, with probabilities defined in \Cref{tab:toy_inlier_spec}. As the level increases, more shapes are equally probable at these three corners, thereby isolating the effect of increasing local multimodality in the inlier distribution. Some example images in each training set are shown in \Cref{fig:fpr_at_tpr_samples}.

\begin{table}[h]
\centering
\caption{\textbf{Toy Experiment 1 Inlier Distribution Specification.}}
\vspace{1mm}
\label{tab:toy_inlier_spec}
\begin{tabular}{lcc}
\toprule
Level & Set of Candidate Shapes & Probability  \\
\midrule
Level 1 & $\{\text{Circle, Square}\}$ & $[0.5,0.5]$ \\
Level 2 & $\{\text{Circle, Square, Triangle}\}$ & $[0.\overline{3},0.\overline{3},0.\overline{3}]$ \\
Level 3 & $\{\text{Circle, Square, Triangle, Ellipse}\}$ & $[0.25,0.25,0.25,0.25]$ \\
Level 4 & $\{\text{Circle, Square, Triangle, Ellipse, Star}\}$ & $[0.2,0.2,0.2,0.2,0.2]$ \\
\bottomrule
\end{tabular}
\end{table}

\begin{figure}[h]
  \centering
  \vspace{-5mm}
  \includegraphics[width=\linewidth]{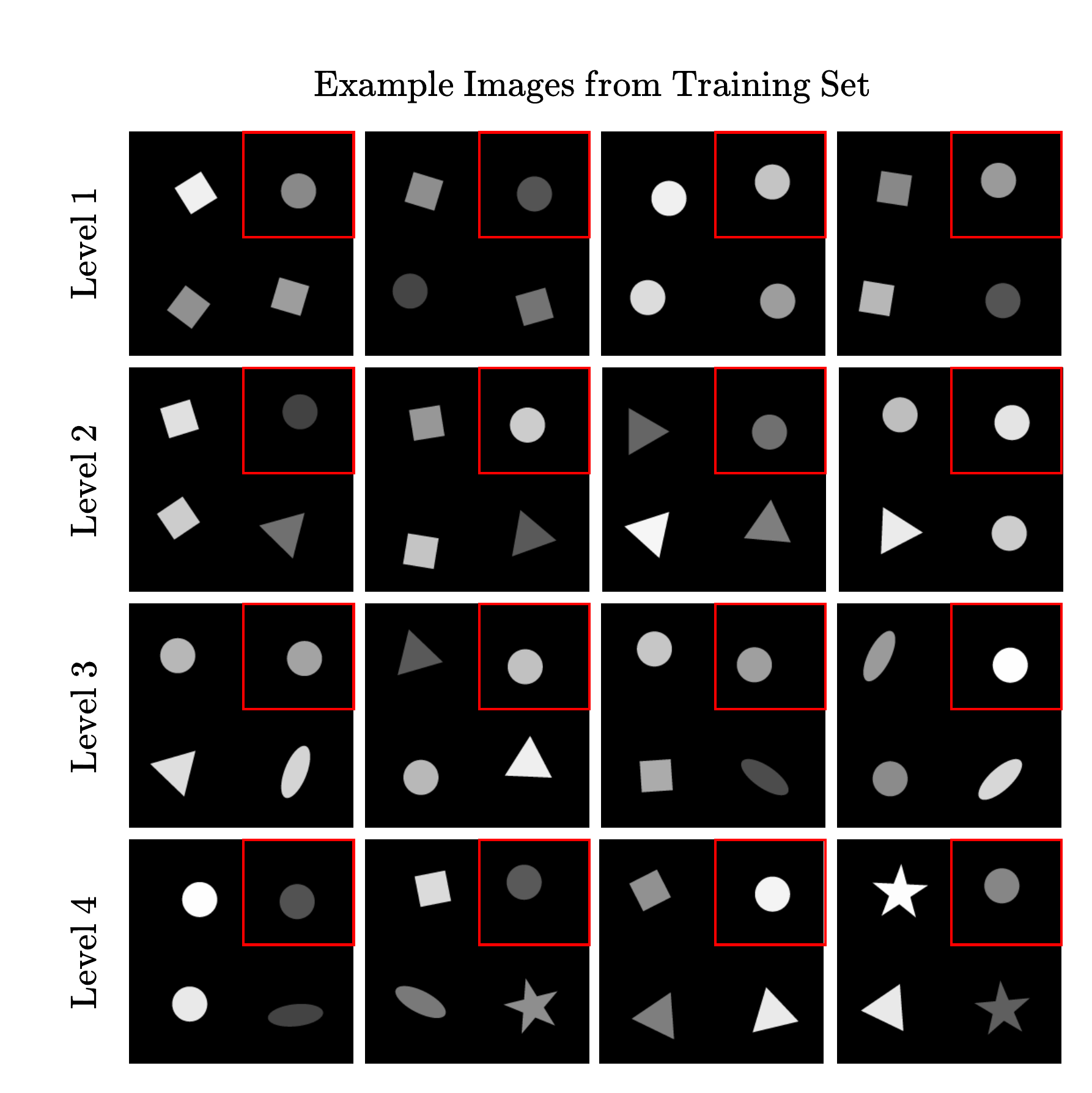}
  \vspace{-8mm}
  \caption{\textbf{Example Images from Training Sets.} The shape orientations are random, with positions sampled with jitter around fixed anchors in each quadrant. The upper-right corner, marked with a red box, always contains a circle. With increasing level, each of the other three corners has increasing modality, as more shapes are equally probable under the inlier distribution.}
  \label{fig:fpr_at_tpr_samples}
\end{figure}

\begin{figure}[t]
  \centering
  \vspace{-5mm}
  \includegraphics[width=\linewidth]{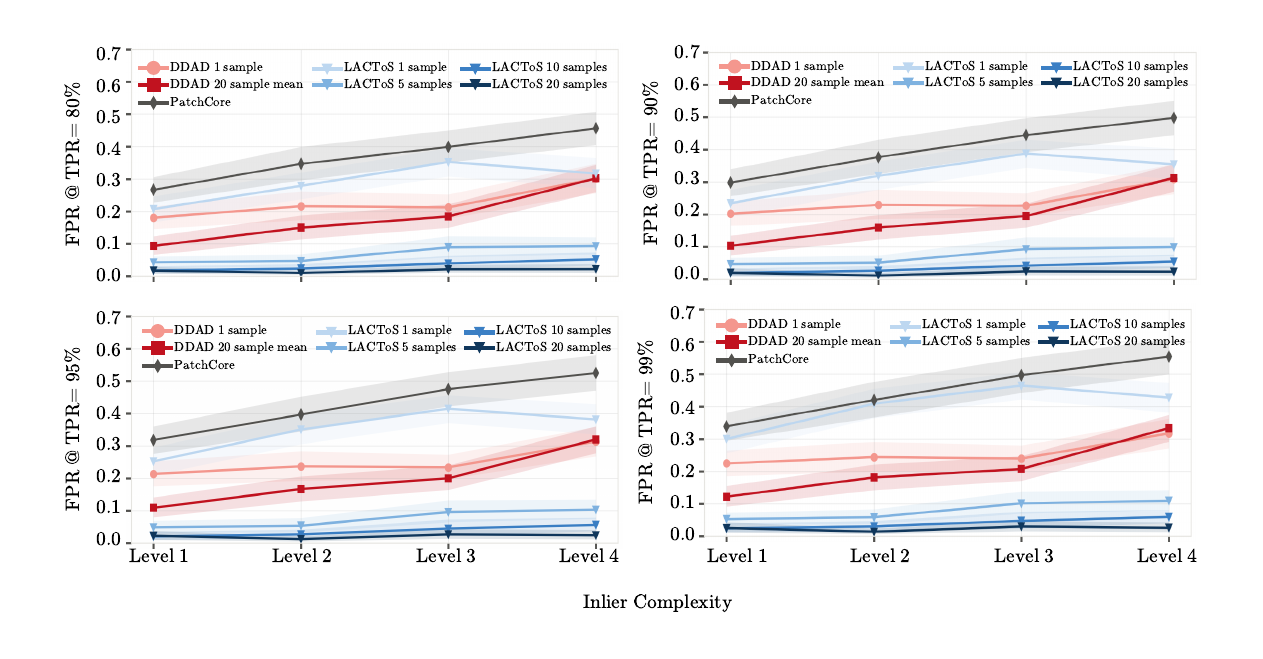}
  \vspace{-10mm}
  \caption{\textbf{\name{} Toy Experiment 1: Robustness Under Multimodality.} Average per-image False Positive Rates at True Positive Rate $\in\{ 80\%,90\%,95\%,99\% \}$, evaluated on images that contain squares on the upper-right corner.}
  \label{fig:fpr_at_tpr}
\end{figure}

Using these diffusion models, we evaluate how the multimodality of the inlier distribution affects the performance of different local anomaly detection methods.
We generate 50 test images that contain a square on the upper-right corner, and place other inlier shapes at the other corners.
We label the pixels in the square as positive (anomalous), and all others as negative (non-anomalous).
For DDAD \cite{ddad}, each test image is noised then denoised, which is used in conjunction with ResNet embeddings of the image to construct the anomaly heatmap. We also evaluate a variant of DDAD that uses more samples by repeating the noising-then-denoising process 20 times with different random seeds, averaging the denoised samples, and using the resulting average to construct the anomaly heatmap.
For PatchCore \cite{patchcore}, a feed-forward model is trained for each inlier distribution, and a heatmap is generated by evaluating the test image with the feed-forward model.
Finally, for \name{}, we follow the process described in \Cref{sec:methods}, while varying the sample count from 1 to 20.
Since the images are grayscale, we use the same set of parameters as for CHAOS, the dataset containing abdominal CT scans which are also in grayscale, for all methods.

To effectively isolate how likely different methods are to erroneously predict inlier shapes as anomalous, we report per-image averaged false- positive rate (FPR) at fixed values of per-image true-positive rate (TPR) $\in\{80\%, 90\%,95\%,99\%\}$.
\Cref{fig:fpr_at_tpr} shows that the FPR of DDAD \citep{ddad} and PatchCore \citep{patchcore} increases with inlier multimodality, whereas the FPR of \name{} decreases rapidly with sample size, and stays negligible compared to the baseline methods when 20 samples are used, even when the inlier distribution is increasingly multimodal.
This demonstrates that \name{} is more robust to inlier multimodality.

\clearpage
\section{Toy Experiment 2}
\label{sec:toy_experiment_2}

We conduct a second toy experiment to better understand how \name{} and baseline methods treat ID but under-represented features, compared to dominant ID features and strictly OOD features.
We hypothesize that, due to the distribution-aware nature of statistical typicality, \name{} considers these minority ID features to be closer to majority ID features, whereas baseline methods that use a first-order statistic may consider them to be closer to OOD.
To demonstrate this concretely, we train 5 diffusion models on training distributions whose specifications are provided in \Cref{tab:toy_2_inlier_spec}.

\begin{table}[h]
\centering
\caption{\textbf{Toy Experiment 2 Inlier Distribution Specification.}}
\vspace{1mm}
\label{tab:toy_2_inlier_spec}
\begin{tabular}{lcc}
\toprule
Level & Set of Candidate Shapes & Probability  \\
\midrule
Level 1 & $\{\text{Circle, Square}\}$ & $\{0.5,0.5\}$ \\
Level 2 & $\{\text{Circle, Square}\}$ & $\{0.6,0.4\}$ \\
Level 3 & $\{\text{Circle, Square}\}$ & $\{0.7,0.3\}$ \\
Level 4 & $\{\text{Circle, Square}\}$ & $\{0.8,0.2\}$ \\
Level 5 & $\{\text{Circle, Square}\}$ & $\{0.9,0.1\}$ \\
\bottomrule
\end{tabular}
\end{table}

Similar to the toy experiment in \Cref{sec:toy_experiment_1}, the upper-right corner still only admits a circle under the inlier distribution, so a square there is a strict OOD feature.
On the other hand, the inlier probability of a square being placed in the other 3 corners is progressively lowered from $0.5$ to $0.1$.
Example images in each training set are shown in \Cref{fig:auroc_samples}.

\begin{figure}[t]
  \centering
  \vspace{-10mm}
  \includegraphics[width=\linewidth]{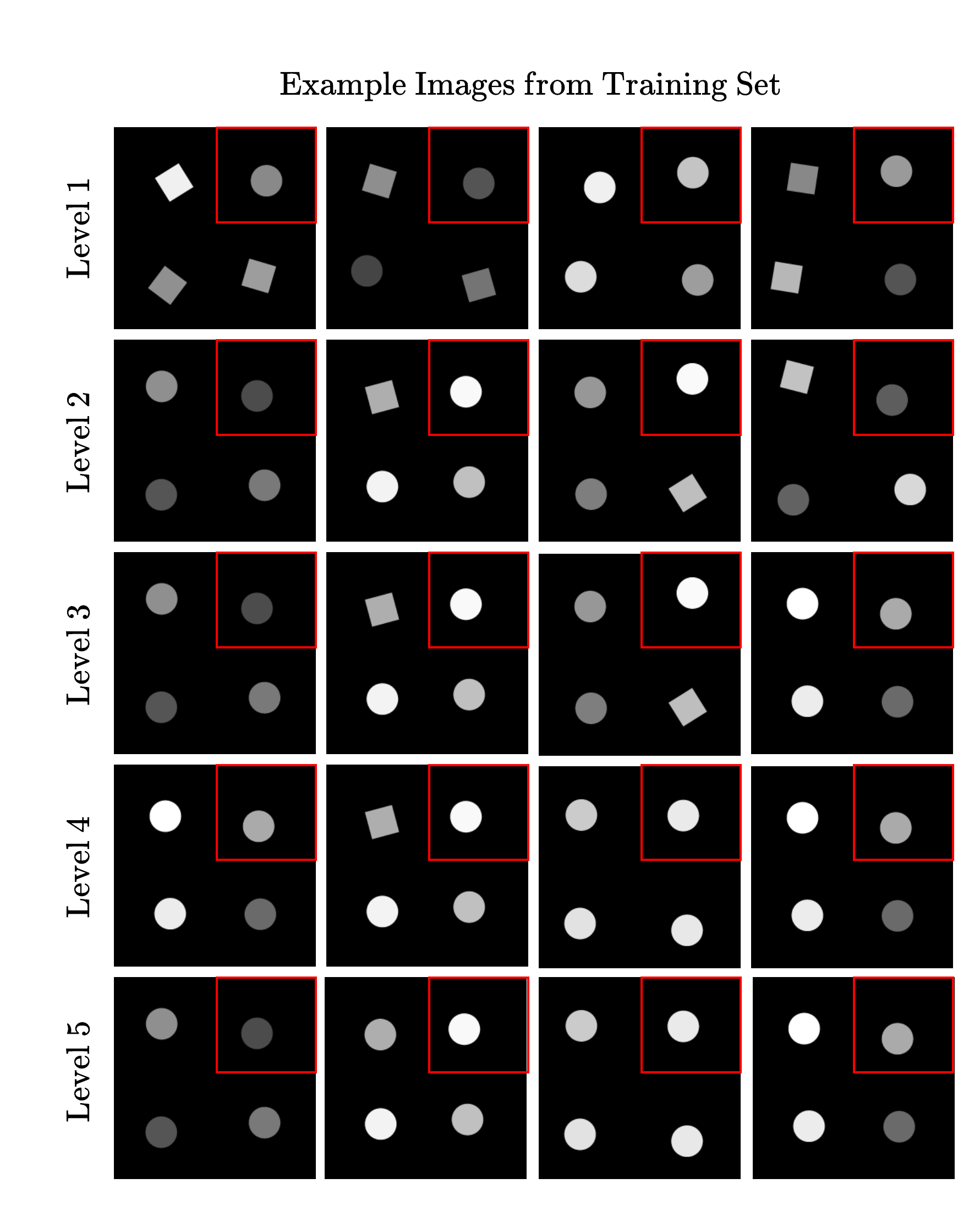}
  \vspace{-8mm}
  \caption{\textbf{Example Images from Training Sets.} Positions of the shapes are sampled with jitter around fixed anchors, with random orientations. The upper-right corner, marked with a red box, always contains a circle. With increasing level, the probability of a square at the other three corners decreases, making it more difficult to distinguish squares there from the square in the upper-right corner.}
  \label{fig:auroc_samples}
\end{figure}

To verify the hypothesis, we measure two quantities.
First, we measure the AUROC over pixels that belong to ID squares and circles, where squares are labeled positive and circles are labeled negative.
A higher AUROC here means the method treats ID, but potentially under-represented, squares differently from ID circles, and vice versa.
Second, we measure the AUROC over pixels that belong to squares, where OOD squares (those at the upper-right corner) are labeled positive and ID squares (those at other corners) are labeled negative.
Higher AUROC here means that the method treats the under-represented ID squares differently from truly OOD squares, and vice versa.
In \Cref{fig:auroc_squares}, we observe that \name{} is less likely to distinguish ID squares from ID circles and more likely to distinguish under-represented ID squares from truly OOD squares compared to DDAD.
Both trends become more apparent with increasing number of sampled reconstructions.
Together, these results verify our hypothesis, and inform practitioners that \name{} is a more reliable choice when one wishes to detect features that are not included in the training set without falsely flagging uncommon inlier features.

\begin{figure}[t]
  \centering
  \vspace{-10mm}
  \includegraphics[width=\linewidth]{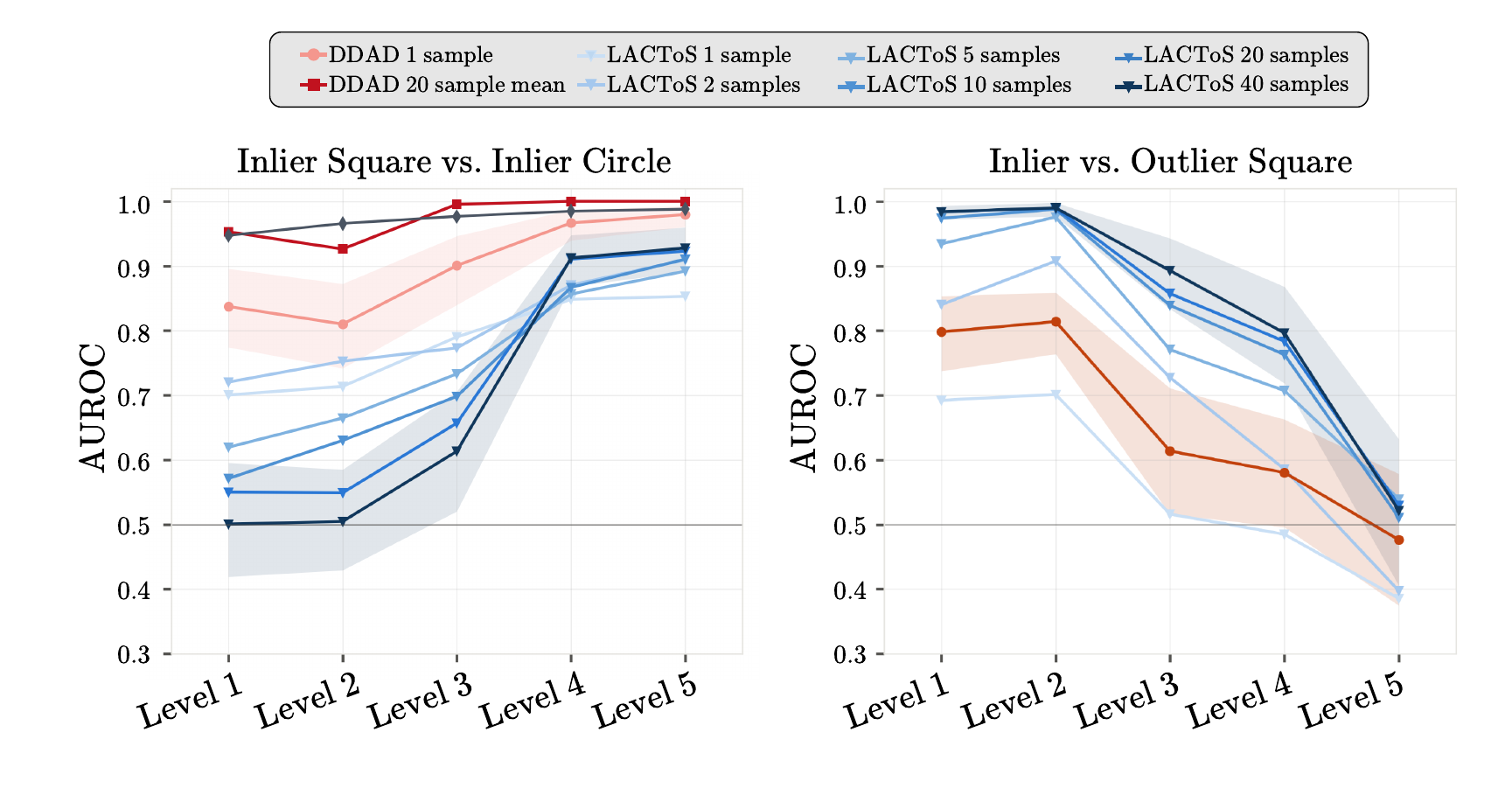}
  \vspace{-8mm}
  \caption{\textbf{\name{} Toy Experiment 2.} AUROCs across different inlier square probabilities, over methods and different sample sizes. (Left) AUROC is computed over pixels that belong to inlier squares and circles, with those pixels from squares labeled positive, and negative otherwise. Higher AUROC means the method is more likely to distinguish an inlier square (minority) from inlier circles (majority). (Right) AUROC is computed over pixels that belong to squares, with those pixels from the square in the upper-right corner (truly OOD) labeled positive, and negative otherwise (minority inlier). Higher AUROC means the method better distinguishes an inlier but uncommon square from a strictly out-of-distribution square.}
  \label{fig:auroc_squares}
\end{figure}

\clearpage
\section{Theory}
\label{app:theory}

Here we provide a formal version of \Cref{thm:main}. Fix a test image \(\xstar\) and a superpixel \(\spxk\in\spxs\).
Let \(\mathcal V:=\Zpixel^C\), where \(\Zpixel=\{0,\ldots,255\}\) denotes the set of possible 8-bit values for one color channel. Thus, \(\mathcal V\) is the single-pixel state space, under the finite quantized histogram setting.
We consider features as an empirical extension, and do not include them in the discussion here.
Under Assumption~\ref{assumption:conditional_independence}, the ideal inpainting posterior \(\pinpk=p(\x\mid\xstarspxnotk)\) induces the single-pixel distribution
\begin{equation}
    \qk
    :=
    p(\xj\mid \xstarspxnotk),
    \qquad
    \qk:\mathcal V\rightarrow[0,1],
    \quad j\in\spxk.
    \label{eq:inp_qk}
    \end{equation}
For \(\eta\sim\mathcal{N}(0,\sigma^2 I)\), write \(y:=\xstar+\eta\).
For a realization \(y^m\) of \(y\), let \(y^m:=\xstar+\eta^m\) denote the noised version of \(\xstar\) used to sample \(\x^m\), and define the pooled single-pixel denoising marginal for superpixel \(\spxk\) as
\begin{equation}
    q_{k,y^m}^{\mathrm{den}}
    :=
    \frac{1}{|\spxk|}
    \sum_{j\in\spxk}
    p(\xj \mid y^m),
    \qquad
    q_{k,y^m}^{\mathrm{den}}:\mathcal V\rightarrow[0,1].
    \label{eq:denoising_qk_y}
\end{equation}
The denoising surrogate distribution used by the score is then
\begin{equation}
    q_k^{\mathrm{den},\sigma}
    :=
    \E_{\eta}\left[q_{k,y}^{\mathrm{den}}\right].
    \label{eq:denoising_qk}
\end{equation}
The corresponding ideal and denoising scores are
\begin{align}
    g^{\mathrm{inp}}(\xstar,k)
    &:=
    \left|
        -\frac{1}{|\spxk|}\sum_{j\in\spxk}\log \qk(\xstarj)
        -|\spxk|\entropy(\qk)
    \right|,
    \label{eq:inp_score}\\
    g^{\mathrm{den}}_{\sigma}(\xstar,k)
    &:=
    \left|
        -\frac{1}{|\spxk|}\sum_{j\in\spxk}\log q_k^{\mathrm{den},\sigma}(\xstarj)
        -|\spxk|\entropy(q_k^{\mathrm{den},\sigma})
    \right|.
    \label{eq:den_score}
\end{align}

For a realization \(y^m\) of \(y\), define the auxiliary distribution
\begin{equation}
    q_{k,y^m}^{\mathrm{aux}}
    :=
    \frac{1}{|\spxk|}
    \sum_{j\in\spxk}
    p(\xj\mid \xstarspxnotk, y^m_{\spxk}),
    \qquad
    q_{k,y^m}^{\mathrm{aux}}:\mathcal V\rightarrow[0,1].
    \label{eq:aux_qk}
\end{equation}
This is the ideal inpainting conditional after additionally revealing the noisy superpixel values, pooled over pixels in \(\spxk\).  It is used only to separate the two sources of error: loss of context from observing the noised version of \(\xstarspxnotk\) instead of \(\xstarspxnotk\), and the leakage of information about \(\xstarspxk\) from observing a noisy instead of completely masked superpixel. We will refer to the former as \emph{context gap}, and the latter as \emph{self-leakage}.

The following assumption formalizes the bias--variance-like tradeoff induced by using denoising instead of inpainting.
A denoising conditional observes a noisy version of the entire image, whereas the ideal inpainting conditional observes the clean complement \(\xstarspxnotk\) and does not observe \(\xstarspxk\).
As \(\sigma\) increases, the noisy observation outside \(\spxk\) becomes a worse proxy for the clean outside context and widens that context gap, which we capture as a linear function of \(\sigma\) in the form of \(\alpha_k\sigma\).
In the opposite direction, when \(\sigma\) is too small, the noisy values inside \(\spxk\) reveal too much about \(\xstarspxk\), allowing the denoising posterior to copy the held-out region rather than predict it from context; this self-leakage is represented by the \(\beta_k\sigma^{-\gamma_k}\) term.

\begin{assumption}
\label{assumption:tradeoff}
[Denoising--inpainting tradeoff]
There exist constants \(\alpha_k,\beta_k,\gamma_k>0\) and \(0<\underline{\sigma}_k\leq\overline{\sigma}_k\) such that, for every \(\sigma\in[\underline{\sigma}_k,\overline{\sigma}_k]\),
\begin{align}
    \E_{\eta}
    \left[
        \left\|
        q_{k,y}^{\mathrm{den}}
        -
        q_{k,y}^{\mathrm{aux}}
        \right\|_1
    \right]
    &\leq \alpha_k\sigma,
    \label{eq:context_gap}\\
    \E_{\eta}
    \left[
        \left\|
        q_{k,y}^{\mathrm{aux}}
        -
        \qk
        \right\|_1
    \right]
    &\leq \beta_k\sigma^{-\gamma_k}.
    \label{eq:self_leakage}
\end{align}
\end{assumption}

We add another assumption, which is a mild regularity condition ensuring that small distributional errors translate into small score errors. 
Since the score contains both a log-likelihood term and an entropy term, probabilities near zero can make the functional unstable.
The lower bound by \(\delta_k\) rules out this degenerate case on the finite single-pixel state space.
In practice, the computed histogram-based scores can enforce such a floor by smoothing or clipping the estimated probabilities.

\begin{assumption}[Score regularity]
\label{assumption:histfloor}
There exists \(\delta_k\in(0,1)\) such that, for every \(\sigma\in[\underline{\sigma}_k,\overline{\sigma}_k]\),
\begin{equation}
    \qk(v)\geq\delta_k
    \quad\text{and}\quad
    q_k^{\mathrm{den},\sigma}(v)\geq\delta_k
    \qquad
    \text{for all }v\in\mathcal V.
    \label{eq:hist_floor}
\end{equation}
\end{assumption}

We now state the formal version of \Cref{thm:main}.

\begin{theorem}[Formal denoising score approximation]
\label{thm:main_formal}
Under Assumptions~\ref{assumption:tradeoff} and~\ref{assumption:histfloor}, for every \(\sigma\in[\underline{\sigma}_k,\overline{\sigma}_k]\),
\begin{equation}
    \left|
        g^{\mathrm{inp}}(\xstar,k)
        -
        g^{\mathrm{den}}_{\sigma}(\xstar,k)
    \right|
    \leq
    L_{\delta_k}\left(\alpha_k\sigma+\beta_k\sigma^{-\gamma_k}\right),
    \label{eq:main_bound}
\end{equation}
where
\begin{equation}
    L_{\delta_k}:=\left(\delta_k^{-1}+1+\log(\delta_k^{-1})\right).
\end{equation}
Equivalently, after absorbing \(L_{\delta_k}\) into the constants, there exist positive constants \(\widetilde \alpha_k\) and \(\widetilde \beta_k\) such that
\begin{equation}
    \left|
        g^{\mathrm{inp}}(\xstar,k)
        -
        g^{\mathrm{den}}_{\sigma}(\xstar,k)
    \right|
    \leq
    \widetilde \alpha_k\sigma+
    \widetilde \beta_k\sigma^{-\gamma_k}.
    \label{eq:main_bound_absorbed}
\end{equation}
\end{theorem}

\begin{proof}
First, we begin by inspecting the error of the anomaly scores.
For any two distributions \(q,r:\mathcal V\rightarrow[0,1]\) satisfying \(q(v),r(v)\geq\delta_k\) for every \(v\in\mathcal V\),
\begin{align}
\label{eq:triangle}
    &\left|
    \left|
        -\frac{1}{|\spxk|}\sum_{j\in\spxk}\log q(\xstarj)-\entropy(q)
    \right|
    -
    \left|
        -\frac{1}{|\spxk|}\sum_{j\in\spxk}\log r(\xstarj)-\entropy(r)
    \right|
    \right| 
    \nonumber\\
    &\qquad\leq
    \left|
        \frac{1}{|\spxk|}\sum_{j\in\spxk}\big(\log r(\xstarj)-\log q(\xstarj)\big)
        +\big(\entropy(r)-\entropy(q)\big)
    \right|
    \nonumber \\
    &\qquad\leq
    \frac{1}{|\spxk|}\sum_{j\in\spxk}
    \left|\log r(\xstarj)-\log q(\xstarj)\right|
    +
    \left|\entropy(q)-\entropy(r)\right|,
\end{align}
by applying reverse triangle inequality, and then triangle inequality.
Since \(t\mapsto -\log t\) has derivative bounded in absolute value by \(\delta_k^{-1}\) on \([\delta_k,1]\),
\begin{equation}
\label{eq:log_lipschitz}
    \frac{1}{|\spxk|}\sum_{j\in\spxk}
    \left|\log r(\xstarj)-\log q(\xstarj)\right|
    \leq
    \delta_k^{-1}\|q-r\|_1.
\end{equation}
Similarly, for \(h(t)=-t\log{t}\), the derivative \(h'(t)=-1-\log{t}\) satisfies
\begin{equation}
    |h'(t)|\leq 1+\log(\delta_k^{-1})
    \qquad
    \text{for }t\in[\delta_k,1].
\end{equation}
Applying the mean-value theorem coordinate-wise gives
\begin{equation}
    \left|\entropy(q)-\entropy(r)\right|
    \leq
    \left(1+\log(\delta_k^{-1})\right)\|q-r\|_1.
    \label{eq:entropy_lipschitz}
\end{equation}
Combining \eqref{eq:triangle}--\eqref{eq:entropy_lipschitz},
\begin{equation}
    \Big|
    \left|
        -\frac{1}{|\spxk|}\sum_{j\in\spxk}\log q(\xstarj)-\entropy(q)
    \right|
    -
    \left|
        -\frac{1}{|\spxk|}\sum_{j\in\spxk}\log r(\xstarj)-\entropy(r)
    \right|
    \Big|
    \leq
    L_{\delta_k}\|q-r\|_1.
    \label{eq:score_lipschitz}
\end{equation}

It remains to control \(\|q_k^{\mathrm{den},\sigma}-\qk\|_1\).
By the definition of \(q_k^{\mathrm{den},\sigma}\) and convexity of the \(\ell_1\) norm,
\begin{align}
    \left\|q_k^{\mathrm{den},\sigma}-\qk\right\|_1
    &=
    \left\|
    \E_{\eta}
    \left[
        q_{k,y}^{\mathrm{den}}
        -
        \qk
    \right]
    \right\|_1
    \nonumber\\
    &\leq
    \E_{\eta}
    \left[
        \left\|
        q_{k,y}^{\mathrm{den}}
        -
        \qk
        \right\|_1
    \right].
    \label{eq:jensen}
\end{align}
For each realization \(y^m\) of \(y\), insert \(q_{k,y^m}^{\mathrm{aux}}\) and apply the triangle inequality:
\begin{align}
    \left\|
        q_{k,y^m}^{\mathrm{den}}
        -
        \qk
    \right\|_1
    &\leq
    \left\|
        q_{k,y^m}^{\mathrm{den}}
        -
        q_{k,y^m}^{\mathrm{aux}}
    \right\|_1
    +
    \left\|
        q_{k,y^m}^{\mathrm{aux}}
        -
        \qk
    \right\|_1.
    \label{eq:error_decomposition}
\end{align}
Taking expectations and using Assumption~\ref{assumption:tradeoff} yields
\begin{equation}
    \left\|q_k^{\mathrm{den},\sigma}-\qk\right\|_1
    \leq
    \alpha_k\sigma+\beta_k\sigma^{-\gamma_k}.
    \label{eq:distribution_bound}
\end{equation}
Finally, apply \eqref{eq:score_lipschitz} with \(q=q_k^{\mathrm{den},\sigma}\) and \(r=\qk\).  
This proves \eqref{eq:main_bound}.  
The absorbed-constant form \eqref{eq:main_bound_absorbed} follows by setting \(\widetilde \alpha_k=L_{\delta_k}\alpha_k\) and \(\widetilde \beta_k=L_{\delta_k}\beta_k\).
\end{proof}

\begin{corollary}[Bound-optimal denoising noise]
\label{corollary:optimal_sigma}
Under the conditions of \Cref{thm:main_formal}, the upper bound in \eqref{eq:main_bound} is minimized over \(\sigma\in[\underline{\sigma}_k,\overline{\sigma}_k]\) by
\begin{equation}
    \sigma_k^{\mathrm{opt}}
    :=
    \min\!\left\{
        \overline{\sigma}_k,
        \max\!\left\{
            \underline{\sigma}_k,
            \left(\frac{\gamma_k \beta_k}{\alpha_k}\right)^{1/(\gamma_k+1)}
        \right\}
    \right\}.
    \label{eq:optimal_sigma}
\end{equation}
In particular, if
\begin{equation}
    \left(\frac{\gamma_k \beta_k}{\alpha_k}\right)^{1/(\gamma_k+1)}
    \in [\underline{\sigma}_k,\overline{\sigma}_k],
\end{equation}
then the bound is minimized at the interior value
\begin{equation}
    \sigma_k^{\mathrm{opt}}
    =
    \left(\frac{\gamma_k \beta_k}{\alpha_k}\right)^{1/(\gamma_k+1)},
    \label{eq:optimal_sigma_interior}
\end{equation}
and the minimum value of the right-hand side of \eqref{eq:main_bound} is
\begin{equation}
    L_{\delta_k}
    (\gamma_k+1)\gamma_k^{-\gamma_k/(\gamma_k+1)}
    \alpha_k^{\gamma_k/(\gamma_k+1)}
    \beta_k^{1/(\gamma_k+1)}.
    \label{eq:optimal_bound_value}
\end{equation}
\end{corollary}

\begin{proof}
The multiplicative factor \(L_{\delta_k}\) in \eqref{eq:main_bound} is positive and therefore does not affect the choice of \(\sigma\).
It is enough to minimize
\begin{equation}
    \alpha_k\sigma+\beta_k\sigma^{-\gamma_k}.
\end{equation}
The derivative is
\begin{equation}
    \alpha_k-\gamma_k \beta_k\sigma^{-\gamma_k-1},
\end{equation}
which vanishes exactly when
\begin{equation}
    \sigma
    =
    \left(\frac{\gamma_k \beta_k}{\alpha_k}\right)^{1/(\gamma_k+1)}.
\end{equation}
The second derivative is
\begin{equation}
    \gamma_k(\gamma_k+1)\beta_k\sigma^{-\gamma_k-2}>0,
\end{equation}
so this critical point is the unique unconstrained minimizer.
Restricting \(\sigma\) to \([\underline{\sigma}_k,\overline{\sigma}_k]\) gives \eqref{eq:optimal_sigma}.
When the unconstrained minimizer lies inside the interval, substituting \eqref{eq:optimal_sigma_interior} into \(L_{\delta_k}(\alpha_k\sigma+\beta_k\sigma^{-\gamma_k})\) gives \eqref{eq:optimal_bound_value}.
\end{proof}

This corollary should be interpreted as selecting the \emph{bound-optimal} noise level rather than the necessarily optimal empirical denoising noise.
The expression balances the two errors in Assumption~\ref{assumption:tradeoff}: increasing \(\sigma\) worsens the context gap, while decreasing \(\sigma\) worsens self-leakage.
At an interior optimum, these two effects satisfy \(\alpha_k\sigma_k^{\mathrm{opt}}=\gamma_k \beta_k(\sigma_k^{\mathrm{opt}})^{-\gamma_k}\).

We now analyze the typicality score without Assumption~\ref{assumption:conditional_independence}. The signed statistic remains centered under the ID model, while dependence affects its concentration through the covariance of the per-pixel information values.

Fix $\spxk$ and $\xstarspxnotk$, and write $n_k:=|\spxk|$. Since the single-pixel marginals may now differ, replace $\qk$ by their pooled marginal
\begin{equation}
    \bar q^{\mathrm{inp}}_k(v)
    :=\frac{1}{n_k}\sum_{j\in\spxk}
      p(\xj=v\mid\xstarspxnotk),
    \qquad v\in\mathcal V.
\label{eq:pooled_inpainting_marginal}
\end{equation}
This reduces to $\qk$ under Assumption~\ref{assumption:conditional_independence}. Retaining the probability floor of Assumption~\ref{assumption:histfloor} for $\bar q^{\mathrm{inp}}_k$, define
\begin{equation}
\begin{aligned}
    Z_{k,j}
    &:= -\log\bar q^{\mathrm{inp}}_k(\xj)
        -\entropy(\bar q^{\mathrm{inp}}_k),\\
    T^{\mathrm{inp}}_k
    &:=\frac{1}{n_k}\sum_{j\in\spxk}Z_{k,j},
    \qquad G^{\mathrm{inp}}_k:=|T^{\mathrm{inp}}_k|.
\end{aligned}
\label{eq:dependent_typicality_score}
\end{equation}
Here $G^{\mathrm{inp}}_k$ is the per-pixel counterpart of the score in \eqref{eq:inp_score}, using the pooled marginal.

By the definition of $\bar q^{\mathrm{inp}}_k$, the expected average information is $\entropy(\bar q^{\mathrm{inp}}_k)$, and hence $\E_{\pinpk}[T^{\mathrm{inp}}_k]=0$. Dependence enters through the total variance
\begin{equation}
\begin{aligned}
    \Gamma^{\mathrm{ID}}_k
    &:=\operatorname{Var}_{\pinpk}
       \left(\sum_{j\in\spxk}Z_{k,j}\right)\\
    &=\sum_{j\in\spxk}\operatorname{Var}_{\pinpk}(Z_{k,j})
      +2\!\sum_{\substack{i,j\in\spxk\\i<j}}
       \operatorname{Cov}_{\pinpk}(Z_{k,i},Z_{k,j}).
\end{aligned}
\label{eq:dependent_id_variance}
\end{equation}
Since $\operatorname{Var}_{\pinpk}(T^{\mathrm{inp}}_k)=\Gamma^{\mathrm{ID}}_k/n_k^2$, Chebyshev's inequality gives
\begin{equation}
    \mathbb P_{\pinpk}(G^{\mathrm{inp}}_k\geq\tau)
    \leq\frac{\Gamma^{\mathrm{ID}}_k}{n_k^2\tau^2},
    \qquad \tau>0.
\label{eq:dependent_id_concentration}
\end{equation}
To express this in terms of effective sample size, let
\begin{equation}
    v^2_{k,\mathrm{ID}}
    :=\frac{1}{n_k}\sum_{j\in\spxk}
      \operatorname{Var}_{\pinpk}(Z_{k,j}),
    \qquad
    n^{\mathrm{ID}}_{k,\mathrm{eff}}
    :=\frac{n_k^2v^2_{k,\mathrm{ID}}}{\Gamma^{\mathrm{ID}}_k}
    \quad(\Gamma^{\mathrm{ID}}_k>0).
\label{eq:dependent_effective_sample_size}
\end{equation}
The bound becomes $v^2_{k,\mathrm{ID}}/(n^{\mathrm{ID}}_{k,\mathrm{eff}}\tau^2)$. Under Assumption~\ref{assumption:conditional_independence}, the cross-covariances vanish and $n^{\mathrm{ID}}_{k,\mathrm{eff}}=n_k$; a positive sum of cross-covariances reduces the effective sample size and weakens concentration.

The same variance dependence appears in OOD detection. For an OOD conditional distribution $R_k$, let $\Delta_k:=|\E_{R_k}[T^{\mathrm{inp}}_k]|$ and $\Gamma^{\mathrm{OOD}}_k:=\operatorname{Var}_{R_k}(\sum_{j\in\spxk}Z_{k,j})$. If $G^{\mathrm{inp}}_k<\tau<\Delta_k$, then $|T^{\mathrm{inp}}_k-\E_{R_k}[T^{\mathrm{inp}}_k]|\geq\Delta_k-\tau$. Chebyshev's inequality therefore yields
\begin{equation}
    \mathbb P_{R_k}(G^{\mathrm{inp}}_k<\tau)
    \leq\frac{\Gamma^{\mathrm{OOD}}_k}
                 {n_k^2(\Delta_k-\tau)^2},
    \qquad 0<\tau<\Delta_k.
\label{eq:dependent_ood_detection}
\end{equation}

Finally, let $G^{\mathrm{den}}_{k,\sigma}$ be the score in \eqref{eq:den_score} divided by $n_k$, with its denoising marginal evaluated at the image being scored. If Assumptions~\ref{assumption:tradeoff} and~\ref{assumption:histfloor} hold uniformly over the considered superpixel realizations, with common constants and $\bar q^{\mathrm{inp}}_k$ replacing $\qk$, the proof of \Cref{thm:main_formal} gives
\begin{equation}
    |G^{\mathrm{inp}}_k-G^{\mathrm{den}}_{k,\sigma}|
    \leq e_k(\sigma),
    \qquad
    e_k(\sigma):=L_{\delta_k}
       \big(\alpha_k\sigma+\beta_k\sigma^{-\gamma_k}\big).
\label{eq:dependent_uniform_approximation}
\end{equation}
Applying \eqref{eq:dependent_id_concentration} and \eqref{eq:dependent_ood_detection} at the resulting shifted thresholds gives
\begin{align}
    \mathbb P_{\pinpk}(G^{\mathrm{den}}_{k,\sigma}\geq\tau)
    &\leq\frac{\Gamma^{\mathrm{ID}}_k}
                  {n_k^2(\tau-e_k(\sigma))^2},
    &&\tau>e_k(\sigma),
\label{eq:dependent_denoising_id_bound}\\
    \mathbb P_{R_k}(G^{\mathrm{den}}_{k,\sigma}<\tau)
    &\leq\frac{\Gamma^{\mathrm{OOD}}_k}
                  {n_k^2(\Delta_k-\tau-e_k(\sigma))^2},
    &&0<\tau<\Delta_k-e_k(\sigma).
\label{eq:dependent_denoising_ood_bound}
\end{align}

Under Assumption~\ref{assumption:conditional_independence}, the effective sample size equals the number of pixels in the superpixel. Without this assumption, the signed statistic remains centered under the ID model, but positive within-superpixel covariance reduces the effective sample size and weakens the concentration and detection bounds. The denoising approximation retains the same form, with the effect of dependence captured by the variance terms.

\section{Experimental Details}
\label{sec:experiment_details}

\subsection{Dataset summary}
Table~\ref{tab:dataset_summary} summarizes the train/test composition,
masking, and image resolution for each dataset. \mvtec{} and \chaos{} are
evaluated at $256\times256$ following standard practice; \ffhq{} is
$256\times256$; \xray{} is $600\times800$ in its native form, and we
report per-pixel mean metrics so resolution does not bias cross-dataset
comparison.

\begin{table}[h]
\centering
\caption{Dataset composition. Training images contain only inlier
(anomaly-free) examples. Test images contain local anomalies and are paired
with pixel-level ground-truth masks.}
\label{tab:dataset_summary}
\begin{tabular}{lcccc}
\toprule
Dataset & \# train (inlier) & \# test (ID + OOD) & Eval mask & Resolution \\
\midrule
\chaos{}           & 3{,}070 &   100 + 100  & body silhouette & $256{\times}256$ \\
\ffhq{}            & 5{,}000 &    50 + 50  & --- & $256{\times}256$ \\
\xray{} (scissors) & 5{,}000 &    72 + 72  & --- & $600{\times}800$ \\
\mvtec{} (15 categories) & 3{,}629 & 1{,}258 & --- & $256{\times}256$ \\
\bottomrule
\end{tabular}
\end{table}

\subsection{Diffusion models and reconstruction}
For \mvtec{}, we use DDAD's checkpoints, which are $\epsilon$-prediction UNet diffusion models, trained separately for each category. For all other benchmarks, we use the guided-diffusion UNet of Dhariwal and Nichol~\citep{guideddiffusion} trained from scratch on each inlier set, and produce reconstructions via DPS~\citep{dps}.

\subsection{\name{} hyperparameters}
For sample reconstruction using DPS \citep{dps}, we utilize the identity forward model and zero-mean Gaussian noise with variance $\sigma_{\mathrm{noise}}$.
DPS also has a scale parameter that controls the posterior guidance strength.

Once $M$ reconstructions are computed and the corresponding features are acquired, we rescale RGB values from $[0,255]$ and
autoencoder outputs from $[0,1]$ to a common $[0,255]$ range.
Then, we partition the reconstructions and feature maps into superpixels containing an average of $S_{\mathrm{avg}}$ pixels each, using marker-controlled watershed segmentation \citep{watershed}.
For each superpixel, we quantize the rescaled values into $B_{\mathrm{RGB}}$ and $B_\mathrm{feat}$ bins per dimension, for RGB values and autoencoder outputs respectively.
While it is ideal to construct $C+C'_f$-dimensional histograms, the number of histogram bins grows exponentially with the dimension.
Therefore, we construct two empirical histograms per superpixel, one over RGB values and one over feature values.
We smooth each histogram with a
Gaussian kernel of $\sigma_{\mathrm{hist}}$ to mitigate empty bins under limited
samples, and use the histograms to compute the RGB and feature typicality scores for each superpixel.
The two scores are summed with a weight parameter $\lambda$ multiplied to the feature typicality to balance the trade-off between the two scores.
Finally, we smooth the anomaly map that contains the same value over all pixels in each superpixel to acquire the pixel-level anomaly map.
All parameters are selected by performing a grid search over a small subset containing 20 test images, and the configuration with the highest AUROC is used to evaluate over the full test set.
The specific parameters used for each benchmark can be found in \Cref{tab:dataset_hyperparameters}.

\begin{table}[h]
  \caption{
  Dataset-specific hyperparameters used for \name{}.
  }
  \label{tab:dataset_hyperparameters}
  \vspace{1mm}
  \centering
  \small
  \setlength{\tabcolsep}{6pt}
  \begin{tabular}{lccc}
    \toprule
    Parameter
    & \chaos{}
    & \ffhq{}
    & \xray{} \\
    \midrule

    DPS $\sigma_{\mathrm{noise}}$
    & 0.1
    & 0.01
    & 2.0 \\

    DPS scale
    & 0.4
    & 0.1
    & 1.0 \\

    $S_{\mathrm{avg}}$
    & 100
    & 20
    & 10 \\

    $M$
    & 5
    & 20
    & 20 \\

    $\sigma_{\mathrm{hist}}$
    & 0.5
    & 1.0
    & 0 \\

    $(B_{\mathrm{RGB}}, B_{\mathrm{feat}})$
    & $(16,64)$
    & $(16,32)$
    & $(16,16)$ \\

    $\lambda$
    & 1.0
    & 0.25
    & 0 \\

    \bottomrule
  \end{tabular}
\end{table}

\subsection{Baselines}
For baselines, we use the official implementations of PatchCore, SimpleNet, SuperSimpleNet, EfficientAD, GLASS, and DDAD, with each method's default backbone:
WideResNet-50 for PatchCore; ImageNet-pretrained ResNet/WideResNet for the SimpleNet variants.

For fair comparison, we perform parameter tuning for each baseline method using the same set of 20 test images.
Note that we do not tune the parameters of SimpleNet, SuperSimpleNet, and GLASS as this would require retraining; we use the default parameters in the respective manuscripts or github repositories.
For PatchCore, $k$, the number of nearest neighbors in the coreset to compare each test patch, is swept across $k\in\{1,2,3,5,7,9,15\}$.
For EfficientAD, $\alpha$, which balances the strengths of the student-teacher detector and the autoencoder detector, is swept across $\alpha\in\{0,0.1,0.2,\cdots,1.0\}$.
For DDAD, a 2D grid of $(w,v)\in\{1,2,4,6,8\}\times\{0.5, 1, 2,4\}$ is searched, where $w$ controls the guidance strength during denoising and $v$ balances the strengths of pixels and extracted features.
The selected parameter values are described in \Cref{tab:dataset_hyperparameters_baselines}.

\begin{table}[h]
  \caption{
  Dataset-specific hyperparameters used for baseline methods.
  }
  \label{tab:dataset_hyperparameters_baselines}
  \vspace{1mm}
  \centering
  \small
  \setlength{\tabcolsep}{6pt}
  \begin{tabular}{lccc}
    \toprule
    Parameter
    & \chaos{}
    & \ffhq{}
    & \xray{} \\
    \midrule

    PatchCore
    & $k=1$
    & $k=1$
    & $k=1$ \\

    EfficientAD
    & $\alpha=0$
    & $\alpha=0.7$
    & $\alpha=0.2$ \\

    DDAD
    & $(w,v)=(2,1)$
    & $(w,v)=(2,4)$
    & $(w,v)=(4,4)$ \\

    \bottomrule
  \end{tabular}
\end{table}

\subsection{Feature extractor and autoencoder}
The feature extractor $\phi$ is an ImageNet-pretrained ResNet-101 \cite{resnet}, frozen at all stages. We extract activations from the last three residual stages, upsample them to the input resolution, and concatenate them. To reduce histogram error under limited samples, these features are compressed to $C'_f = 3$ dimensions by a per-pixel autoencoder trained on each inlier dataset.

\clearpage
\section{Additional Results}

We provide additional visualizations of each benchmark in \Cref{fig:chaos_more}, \Cref{fig:ffhq_more}, and \Cref{fig:dvxray_more}.

\begin{figure}[h]
  \centering
  \includegraphics[width=\linewidth]{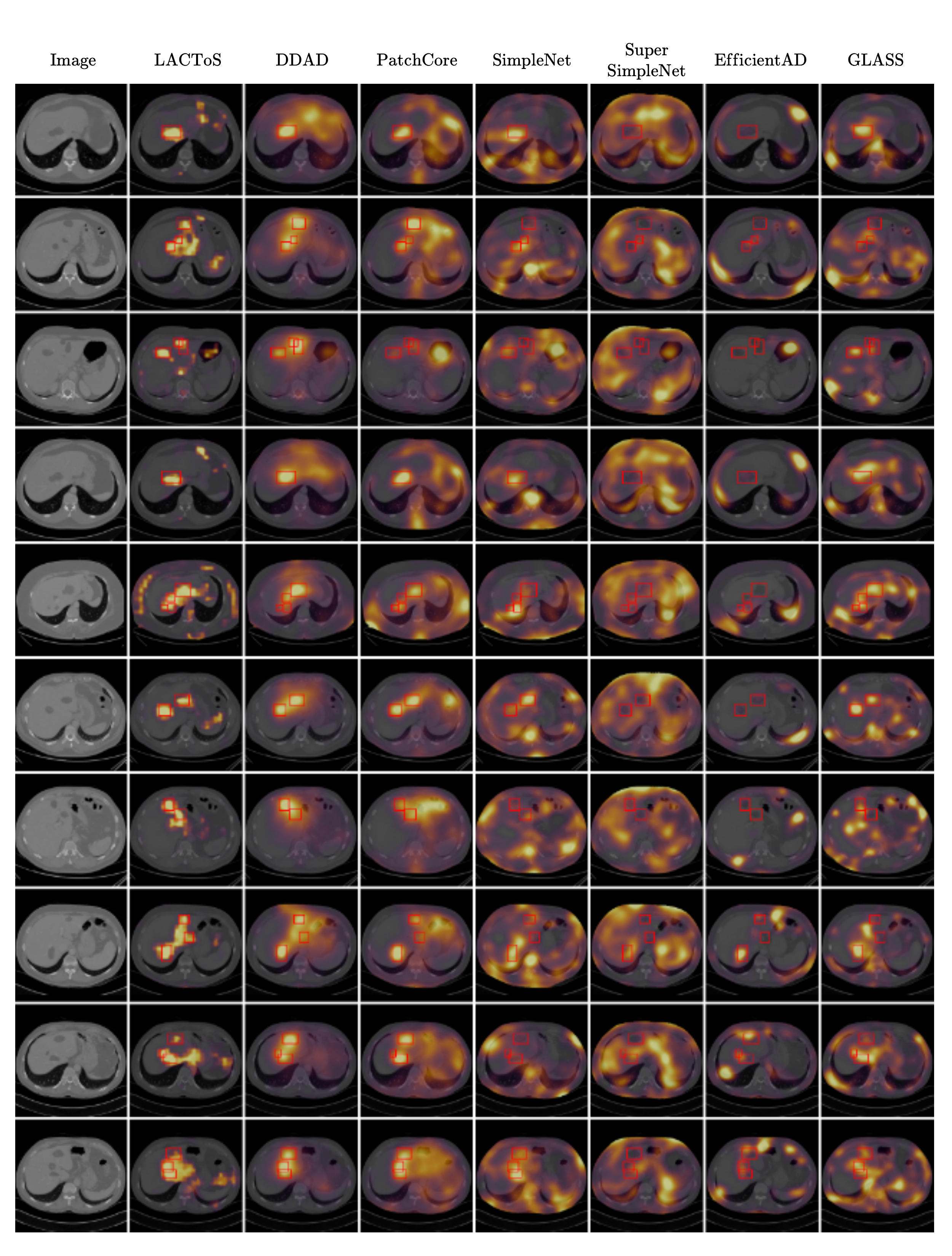}
  \vspace{-8mm}
  \caption{\textbf{\chaos{} Additional Visualizations.} 10 randomly selected test images and the heatmaps for each method.}
  \label{fig:chaos_more}
\end{figure}

\begin{figure}[h]
  \centering
  \includegraphics[width=\linewidth]{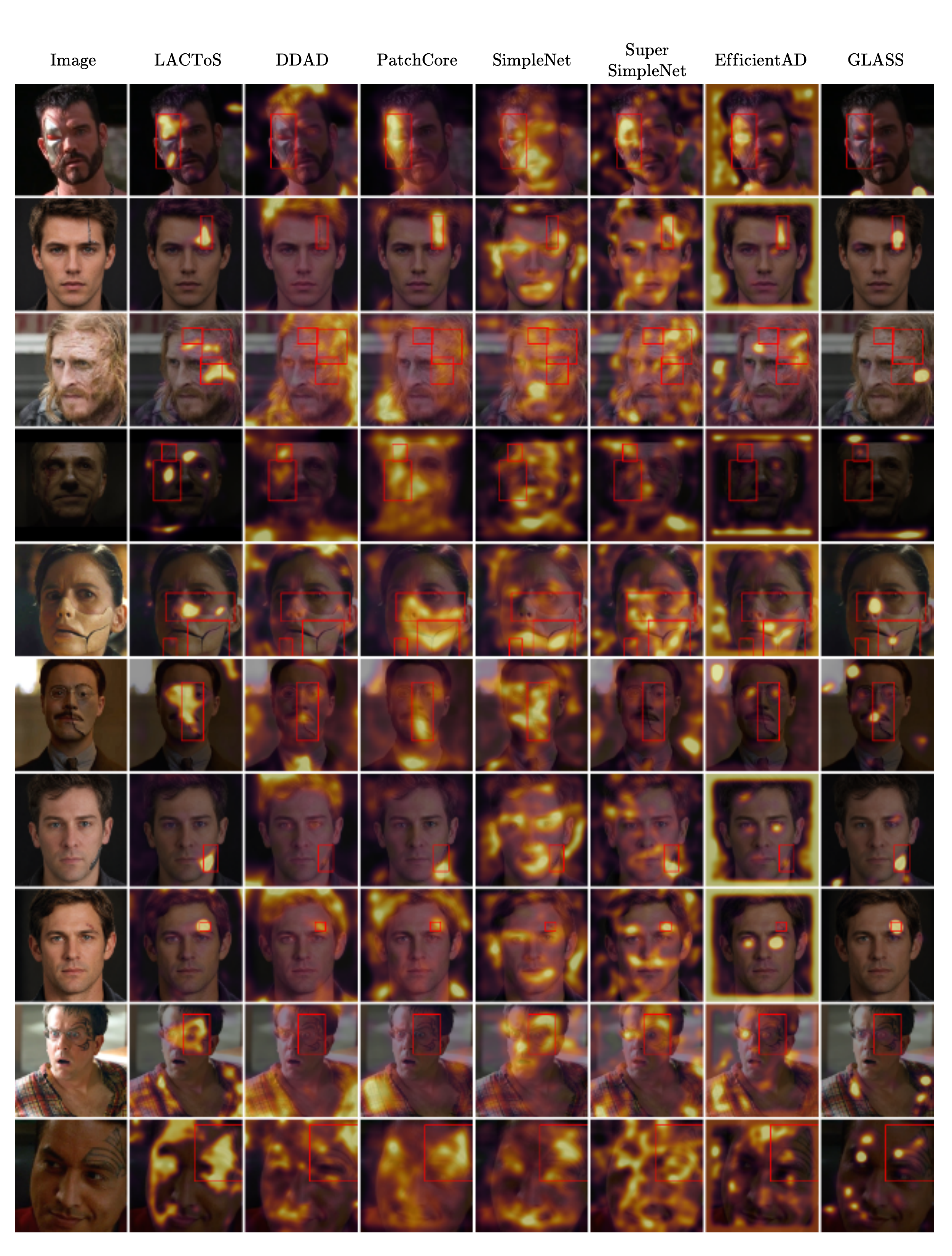}
  \vspace{-8mm}
  \caption{\textbf{\ffhq{} Additional Visualizations.} 10 randomly selected test images and the heatmaps for each method.}
  \label{fig:ffhq_more}
\end{figure}

\begin{figure}[h]
  \centering
  \includegraphics[width=\linewidth]{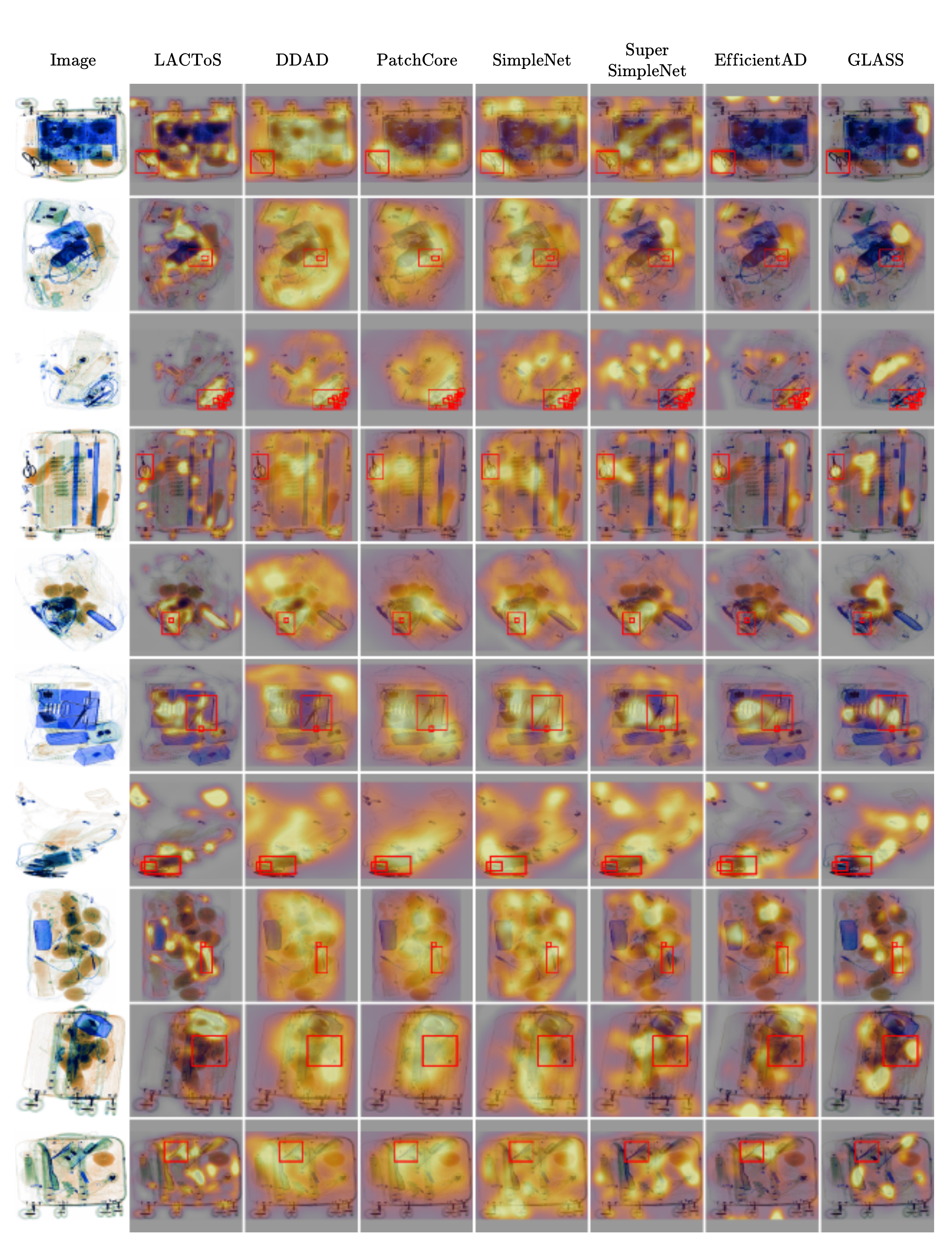}
  \vspace{-8mm}
  \caption{\textbf{\xray{} Additional Visualizations.} 10 randomly selected test images and the heatmaps for each method.}
  \label{fig:dvxray_more}
\end{figure}

\clearpage

In \Cref{tab:app_mvtec}, we report the AUROC, AP, and AUPRO of all methods for each of the object category in the MVTec dataset.

\begin{table*}[h]
  \centering
  \caption{
  \textbf{Results on MVTec AD.}
  We report pixel-level AUROC, AP, and AUPRO for each object or texture category.
  AUROC and AP pool every pixel of all test images of the category (anomalous and anomaly-free);
  AUPRO is the region-based per-region overlap integrated up to an FPR of 0.3.
  Each method occupies three rows, one per metric.
  Higher is better for all metrics.
  }
  \label{tab:app_mvtec}
  \vspace{1mm}
  \small
  \setlength{\tabcolsep}{4.5pt}
  \renewcommand{\arraystretch}{1.0}

  \resizebox{\linewidth}{!}{%
    \begin{tabular}{llcccccccc}
      \specialrule{1.5pt}{0.0ex}{0.75ex}
      Method & Metric
        & Carpet & Grid & Leather & Tile
        & Wood & Bottle & Cable & Capsule \\
      \midrule
      \multirow{3}{*}{SimpleNet}
        & AUROC & 0.9765 & 0.9656 & 0.9898 & 0.9326 & 0.9177 & 0.9775 & 0.9688 & 0.9862 \\
        & AP    & 0.3650 & 0.2440 & 0.3257 & 0.5470 & 0.3673 & 0.6315 & 0.6151 & 0.3729 \\
        & AUPRO & 0.9047 & 0.8757 & 0.9630 & 0.7781 & 0.7866 & 0.8945 & 0.8741 & 0.9093 \\
      \midrule
      \multirow{3}{*}{SuperSimpleNet}
        & AUROC & 0.9832 & 0.9771 & 0.9878 & 0.9383 & 0.9208 & 0.9779 & 0.9728 & 0.9805 \\
        & AP    & 0.4548 & 0.2383 & 0.3492 & 0.5109 & 0.3867 & 0.6202 & 0.6151 & 0.3973 \\
        & AUPRO & 0.9222 & 0.9236 & 0.9686 & 0.8107 & 0.8551 & 0.8982 & 0.8746 & 0.8753 \\
      \midrule
      \multirow{3}{*}{PatchCore}
        & AUROC & 0.9877 & 0.9729 & 0.9903 & 0.9496 & 0.9301 & 0.9813 & 0.9831 & 0.9873 \\
        & AP    & 0.5770 & 0.2326 & 0.3533 & 0.4947 & 0.4016 & 0.6842 & 0.6258 & 0.3935 \\
        & AUPRO & 0.9295 & 0.8932 & 0.9574 & 0.8212 & 0.8312 & 0.9253 & 0.9153 & 0.9227 \\
      \midrule
      \multirow{3}{*}{DDAD}
        & AUROC & 0.9822 & 0.9795 & 0.9878 & 0.9790 & 0.9374 & 0.9860 & 0.9784 & 0.8693 \\
        & AP    & 0.6422 & 0.2710 & 0.4124 & 0.7853 & 0.5285 & 0.8047 & 0.6641 & 0.2313 \\
        & AUPRO & 0.8994 & 0.9291 & 0.9811 & 0.9117 & 0.8604 & 0.9003 & 0.8835 & 0.6845 \\
      \midrule
      \multirow{3}{*}{EfficientAD}
        & AUROC & 0.9831 & 0.9833 & 0.9916 & 0.9816 & 0.9680 & 0.9884 & 0.9863 & 0.9915 \\
        & AP    & 0.7511 & 0.4328 & 0.5421 & 0.7505 & 0.6273 & 0.7852 & 0.7593 & 0.4603 \\
        & AUPRO & 0.9489 & 0.9144 & 0.9840 & 0.9249 & 0.9207 & 0.9550 & 0.9009 & 0.9445 \\
      \midrule
      \multirow{3}{*}{GLASS}
        & AUROC & 0.9952 & 0.9908 & 0.9952 & 0.9963 & 0.9761 & 0.9922 & 0.9569 & 0.9898 \\
        & AP    & 0.8108 & 0.4528 & 0.5518 & 0.9640 & 0.7196 & 0.8966 & 0.4997 & 0.4798 \\
        & AUPRO & 0.9832 & 0.9606 & 0.9815 & 0.9770 & 0.9467 & 0.9652 & 0.8470 & 0.9460 \\
      \midrule
      \multirow{3}{*}{\name{} (Ours)}
        & AUROC & 0.9822 & 0.9712 & 0.9935 & 0.9307 & 0.9099 & 0.9839 & 0.9739 & 0.9703 \\
        & AP    & 0.4489 & 0.3813 & 0.6339 & 0.7085 & 0.4139 & 0.8151 & 0.5891 & 0.4118 \\
        & AUPRO & 0.9209 & 0.8934 & 0.9540 & 0.8165 & 0.6801 & 0.8877 & 0.8712 & 0.8757 \\

      \specialrule{1.5pt}{0.75ex}{0.75ex}

      Method & Metric
        & Hazelnut & Metal Nut & Pill & Screw
        & Toothbrush & Transistor & Zipper & Average \\
      \midrule
      \multirow{3}{*}{SimpleNet}
        & AUROC & 0.9741 & 0.9747 & 0.9839 & 0.9887 & 0.9823 & 0.9746 & 0.9790 & 0.9715 \\
        & AP    & 0.3978 & 0.8049 & 0.7402 & 0.2325 & 0.3377 & 0.6181 & 0.4712 & 0.4714 \\
        & AUPRO & 0.9297 & 0.8082 & 0.9009 & 0.9464 & 0.8404 & 0.9354 & 0.9091 & 0.8837 \\
      \midrule
      \multirow{3}{*}{SuperSimpleNet}
        & AUROC & 0.9792 & 0.9855 & 0.9841 & 0.9873 & 0.9851 & 0.9618 & 0.9813 & 0.9735 \\
        & AP    & 0.5013 & 0.8630 & 0.7763 & 0.2242 & 0.4923 & 0.6081 & 0.4395 & 0.4985 \\
        & AUPRO & 0.9447 & 0.9004 & 0.9326 & 0.9434 & 0.8473 & 0.9130 & 0.9304 & 0.9027 \\
      \midrule
      \multirow{3}{*}{PatchCore}
        & AUROC & 0.9851 & 0.9659 & 0.9895 & 0.9908 & 0.9827 & 0.9742 & 0.9813 & 0.9768 \\
        & AP    & 0.5019 & 0.7060 & 0.7824 & 0.2585 & 0.3272 & 0.6312 & 0.4444 & 0.4943 \\
        & AUPRO & 0.9559 & 0.8557 & 0.9517 & 0.9572 & 0.8613 & 0.9431 & 0.9244 & 0.9097 \\
      \midrule
      \multirow{3}{*}{DDAD}
        & AUROC & 0.9832 & 0.9840 & 0.9461 & 0.8490 & 0.9729 & 0.9436 & 0.9751 & 0.9569 \\
        & AP    & 0.5405 & 0.8580 & 0.6220 & 0.0202 & 0.5126 & 0.6341 & 0.5388 & 0.5377 \\
        & AUPRO & 0.9279 & 0.9072 & 0.5272 & 0.5060 & 0.6514 & 0.9233 & 0.9056 & 0.8266 \\
      \midrule
      \multirow{3}{*}{EfficientAD}
        & AUROC & 0.9856 & 0.9847 & 0.9885 & 0.9945 & 0.9882 & 0.9775 & 0.9892 & 0.9855 \\
        & AP    & 0.5492 & 0.8220 & 0.7474 & 0.4125 & 0.5202 & 0.8170 & 0.5594 & 0.6358 \\
        & AUPRO & 0.9457 & 0.9337 & 0.9606 & 0.9692 & 0.9315 & 0.9455 & 0.9615 & 0.9427 \\
      \midrule
      \multirow{3}{*}{GLASS}
        & AUROC & 0.9863 & 0.8987 & 0.9727 & 0.9315 & 0.9629 & 0.5712 & 0.9948 & 0.9474 \\
        & AP    & 0.6267 & 0.5778 & 0.8330 & 0.2391 & 0.2321 & 0.1719 & 0.7725 & 0.5886 \\
        & AUPRO & 0.9733 & 0.9004 & 0.8555 & 0.7053 & 0.7514 & 0.3643 & 0.9822 & 0.8760 \\
      \midrule
      \multirow{3}{*}{\name{} (Ours)}
        & AUROC & 0.9892 & 0.9871 & 0.9777 & 0.9916 & 0.9871 & 0.9851 & 0.9535 & 0.9725 \\
        & AP    & 0.8157 & 0.9177 & 0.7423 & 0.4611 & 0.6775 & 0.8292 & 0.2678 & 0.6076 \\
        & AUPRO & 0.9190 & 0.8397 & 0.9110 & 0.9663 & 0.8534 & 0.8612 & 0.8247 & 0.8717 \\

      \bottomrule
    \end{tabular}%
  }
\end{table*}

\clearpage
\section{Ablation Studies}
\label{sec:ablations}

In this section, we ablate our method on the different choices of hyper-parameters. 

\paragraph{Sample efficiency.}
\Cref{tab:samplesize} reports the AUROC, AP, and AUPRO of diffusion-based methods DDAD \citep{ddad} and \name{} as the number of denoising samples $\M$ varies from 1 to 20.
\name{} exhibits improved performance with increasing $\M$ on \chaos{} and \xray{} compared to DDAD, which relies on point estimates of reconstruction errors.
This is consistent with the statistical properties of typicality: more samples yield better estimations of $\hat{q}_k^{\mathrm{inp}}$, which improves the reliability of the entropy and log likelihood estimates.
At $\M=1$, \name{} degenerates toward a reconstruction-error-like score.

\paragraph{Superpixel size.}
\Cref{tab:ablation_num_superpixels} shows the effect of the number of superpixels' average size in the performance of \name{}, which is mostly stable but with different trends across benchmarks.

\paragraph{AutoEncoder features.}
The sensitivity of the metrics under varying $\lambda$ is illustrated in \Cref{tab:ablation_pca_features}. $\lambda=0$ corresponds to using RGB typicality only. In most datasets, including feature typicality significantly improves performance.

\begin{table}[ht!]
  \caption{Ablation on sample size $M$ in diffusion-based methods. AUROC, AP, and AUPRO are reported for each dataset.}
  \label{tab:samplesize}
  \vspace{1mm}
  \centering
  \small
  \setlength{\tabcolsep}{4pt}
  \begin{tabular}{llcccccccc}
    \toprule
    \multirow{2}{*}{Dataset} & \multirow{2}{*}{Metric}
      & \multicolumn{2}{c}{$\M=1$} & \multicolumn{2}{c}{$\M=5$}
      & \multicolumn{2}{c}{$\M=10$} & \multicolumn{2}{c}{$\M=20$} \\
    \cmidrule(lr){3-4} \cmidrule(lr){5-6} \cmidrule(lr){7-8} \cmidrule(lr){9-10}
      & & DDAD & \name{} & DDAD & \name{} & DDAD & \name{} & DDAD & \name{} \\
    \midrule
    \multirow{3}{*}{\chaos{}}
      & AUROC & 0.9707 & 0.9552 & 0.9748 & 0.9620 & 0.9749 & 0.9615 & 0.9745 & 0.9620 \\
      & AP    & 0.3159 & 0.2536 & 0.3496 & 0.4318 & 0.3468 & 0.4236 & 0.3489 & 0.4338 \\
      & AUPRO & 0.8676 & 0.8524 & 0.8827 & 0.8764 & 0.8825 & 0.8713 & 0.8814 & 0.8729 \\
    \midrule
    \multirow{3}{*}{\ffhq{}}
      & AUROC & 0.7469 & 0.8446 & 0.7512 & 0.8521 & 0.7559 & 0.8503 & 0.7583 & 0.8515 \\
      & AP    & 0.0563 & 0.0937 & 0.0543 & 0.1197 & 0.0552 & 0.1119 & 0.0564 & 0.1177 \\
      & AUPRO & 0.3083 & 0.4638 & 0.3065 & 0.5157 & 0.3202 & 0.5239 & 0.3151 & 0.5320 \\
    \midrule
    \multirow{3}{*}{\xray{}}
      & AUROC & 0.8567 & 0.9347 & 0.8518 & 0.9541 & 0.8523 & 0.9585 & 0.8495 & 0.9607 \\
      & AP    & 0.0118 & 0.0248 & 0.0119 & 0.0465 & 0.0119 & 0.0530 & 0.0118 & 0.0568 \\
      & AUPRO & 0.5183 & 0.7385 & 0.4896 & 0.7964 & 0.4880 & 0.8055 & 0.4890 & 0.8179 \\
    \bottomrule
  \end{tabular}
\end{table}

\begin{table}[ht!]
  \caption{Ablation on the average size of superpixels. The larger the average size, the smaller the total number of superpixels. AUROC, AP, and AUPRO on \chaos{}, \ffhq{}, and \xray{}.}
  \label{tab:ablation_num_superpixels}
  \vspace{1mm}
  \centering
  \begin{tabular}{llccccccc}
    \toprule
    \multirow{2}{*}{Dataset} & \multirow{2}{*}{Metric}
      & \multicolumn{7}{c}{Average size of superpixels} \\
    \cmidrule(lr){3-8}
      & & $10$ & $20$ & $30$ & $50$ & $70$ & $100$ & $200$ \\
    \midrule
        \multirow{3}{*}{\chaos{}}
      & AUROC & $0.9418$ & $0.9504$ & $0.9535$ & $0.9578$ & $0.9595$ & $0.9620$ & $0.9593$ \\
      & AP    & $0.1271$ & $0.1599$ & $0.1857$ & $0.2434$ & $0.2643$ & $0.4318$ & $0.3677$ \\
      & AUPRO & $0.7992$ & $0.8274$ & $0.8385$ & $0.8499$ & $0.8571$ & $0.8764$ & $0.8600$ \\
    \midrule
    \multirow{3}{*}{\ffhq{}}
      & AUROC & $0.7930$ & $0.8515$ & $0.8502$ & $0.8480$ & $0.8462$ & $0.8438$ & $0.8383$ \\
      & AP    & $0.0774$ & $0.1177$ & $0.1145$ & $0.1107$ & $0.1095$ & $0.1065$ & $0.0998$ \\
      & AUPRO & $0.3890$ & $0.5320$ & $0.5263$ & $0.5180$ & $0.5101$ & $0.5068$ & $0.4863$ \\
    \midrule
    \multirow{3}{*}{\xray{}}
      & AUROC & $0.9607$ & $0.9611$ & $0.9593$ & $0.9501$ & $0.9468$ & $0.9463$ & $0.9330$ \\
      & AP    & $0.0568$ & $0.0570$ & $0.0537$ & $0.0495$ & $0.0458$ & $0.0418$ & $0.0342$ \\
      & AUPRO & $0.8179$ & $0.8122$ & $0.8125$ & $0.8040$ & $0.7872$ & $0.7817$ & $0.7466$ \\

    \bottomrule
  \end{tabular}
\end{table}

\begin{table}[t!]
  \caption{
  Ablation on the feature-weighting parameter $\lambda$.
  $\lambda=0$ corresponds to using RGB features only.
  AUROC, AP, and AUPRO are reported on \chaos{}, \ffhq{}, and \xray{}.
  }
  \label{tab:ablation_pca_features}
  \vspace{1mm}
  \centering
  \small
  \setlength{\tabcolsep}{4pt}
  \renewcommand{\arraystretch}{1.1}
  \resizebox{\linewidth}{!}{%
  \begin{tabular}{llccccccccc}
    \toprule
    Dataset & Metric
    & $\lambda=0$
    & $\lambda=0.25$
    & $\lambda=0.5$
    & $\lambda=0.75$
    & $\lambda=1.0$
    & $\lambda=1.5$
    & $\lambda=2$
    & $\lambda=4$
    & Feature only \\
    \midrule

    \multirow{3}{*}{\chaos{}}
      & AUROC & 0.8418 & 0.9638 & 0.9632 & 0.9625 & 0.9620 & 0.9613 & 0.9609 & 0.9601 & 0.9592 \\
      & AP    & 0.0729 & 0.3893 & 0.4318 & 0.4356 & 0.4318 & 0.4192 & 0.4056 & 0.3645 & 0.2733 \\
      & AUPRO & 0.5689 & 0.8764 & 0.8778 & 0.8772 & 0.8764 & 0.8755 & 0.8748 & 0.8735 & 0.8717 \\
    \midrule
    \multirow{3}{*}{\ffhq{}}
      & AUROC & 0.7605 & 0.8515 & 0.8488 & 0.8459 & 0.8437 & 0.8406 & 0.8386 & 0.8347 & 0.8293 \\
      & AP    & 0.0677 & 0.1177 & 0.1115 & 0.1075 & 0.1041 & 0.0997 & 0.0975 & 0.0944 & 0.0916 \\
      & AUPRO & 0.3688 & 0.5320 & 0.5353 & 0.5345 & 0.5331 & 0.5305 & 0.5287 & 0.5252 & 0.5211 \\
    \midrule
    \multirow{3}{*}{\xray{}}
      & AUROC & 0.9607 & 0.9589 & 0.9522 & 0.9454 & 0.9397 & 0.9305 & 0.9235 & 0.9057 & 0.8088 \\
      & AP    & 0.0568 & 0.0535 & 0.0425 & 0.0336 & 0.0285 & 0.0232 & 0.0205 & 0.0165 & 0.0115 \\
      & AUPRO & 0.8179 & 0.7980 & 0.7686 & 0.7432 & 0.7223 & 0.6905 & 0.6675 & 0.6122 & 0.4791 \\

    \bottomrule
  \end{tabular}%
  }
\end{table}

\end{document}